\documentclass[acmsmall, nonacm]{acmart}

\renewcommand\footnotetextcopyrightpermission[1]{}

\usepackage[english]{babel}
\usepackage{blindtext}

\usepackage{amsmath,amsfonts,bm}

\def\eqref#1{equation~\ref{#1}}

\def\1{\bm{1}}

\DeclareMathAlphabet{\mathsfit}{\encodingdefault}{\sfdefault}{m}{sl}
\SetMathAlphabet{\mathsfit}{bold}{\encodingdefault}{\sfdefault}{bx}{n}

\newcommand{\E}{\mathbb{E}}

\usepackage{hyperref}
\usepackage{url}

\usepackage{upgreek}

\newcommand{\ea}{{et al.}\xspace}

\newcommand{\parab}[1]{\vspace{0.05in}\noindent\textbf{#1}}

\newcommand{\mycomment}[1]{}

\usepackage{enumitem}
\usepackage{pbox}
\usepackage{multirow}
\usepackage{tabularx}
\usepackage{makecell}
\usepackage{mathtools}
\usepackage{caption}
\usepackage{float}
\usepackage{booktabs}
\usepackage{xcolor}
\usepackage{xspace}
\usepackage[labelformat=empty]{subcaption}

\makeatletter
\newcommand{\labeltext}[2]{%
  \@bsphack
  \csname phantomsection\endcsname 
  \def\@currentlabel{#1}{\label{#2}}%
  \@esphack
}
\makeatother

\begin{document}
\title[\textsc{Plume}: A Framework for High Performance Deep RL Network Controllers]{\textsc{Plume}: A Framework for High Performance Deep RL Network Controllers via Prioritized Trace Sampling}


\author{\small Sagar Patel\textsuperscript{1}}
\author{\small Junyang Zhang\textsuperscript{1}}
\author{\small Sangeetha Abdu Jyothi\textsuperscript{1, 2}} 
\author{\small Nina Narodytska\textsuperscript{2}}
\affiliation{
    \institution{\\\textsuperscript{1}University of Califoria, Irvine}
    \city{Irvine}
    \state{CA}
    \country{USA}
}
\affiliation{
    \institution{\textsuperscript{2}VMware Research}
    \city{Palo Alto}
    \state{CA}
    \country{USA}
    }
\email{sagar.patel@uci.edu}


\renewcommand{\shortauthors}{Sagar Patel, Junyang Zhang, Sangeetha Abdu Jyothi, and Nina Narodytska}

\begin{abstract}

Deep Reinforcement Learning (DRL) has shown promise in various networking environments. However, these environments present several fundamental challenges for standard DRL techniques. They are difficult to explore and exhibit high levels of noise and uncertainty. Although these challenges complicate the training process, we find that in practice we can substantially mitigate their effects and even achieve state-of-the-art real-world performance by addressing a factor that has been previously overlooked: the skewed input trace distribution in DRL training datasets. 

We introduce a generalized framework, \textit{Plume}, to automatically identify and balance the skew using a three-stage process. First, we identify the critical features that determine the behavior of the traces. Second, we classify the traces into clusters. Finally, we prioritize the salient clusters to improve the overall performance of the controller. Plume seamlessly works across DRL algorithms, without requiring any changes to the DRL workflow. We evaluated Plume on three networking environments, including Adaptive Bitrate Streaming, Congestion Control, and Load Balancing. Plume offers superior performance in both simulation and real-world settings, across different controllers and DRL algorithms. For example, our novel ABR controller, \textit{Gelato} trained with Plume consistently outperforms prior state-of-the-art controllers on the live streaming platform Puffer for over a year. It is the first controller on the platform to deliver statistically significant improvements in both video quality and stalling, decreasing stalls by as much as $75\%$.

\end{abstract}

\maketitle

\section{Introduction}

Control in real-world networks is a hard-to-tackle problem. Today, learning solutions hold promise in a broad range of network environments. Deep Reinforcement Learning (DRL) controllers have shown encouraging results in congestion control (CC)~\cite{aurora, orca}, Adaptive Bitrate Streaming (ABR)~\cite{pensieve, pensieveRealWorld}, load balancing (LB)~\cite{mao2018variance}, cluster scheduling~\cite{decima}, network traffic optimization~\cite{auto} and network planning~\cite{neuroplan}, to name a few. 

In systems and networking environments, unlike traditional DRL environments such as gaming or robotics, there exists an unpredictable underlying input process. For instance, in congestion control, the behavior of other traffic sharing the network path determines if congestion occurs. These processes are referred to as ``inputs''~\cite{mao2018variance}. During training, they are replayed using a dataset of \textit{input traces}, or system logs. Such input-driven DRL environments have several characteristics that make DRL training difficult. 
First, input-driven RL environments require more extensive exploration during training~\cite{mao2019park}. Second, the dependence of the environment on external inputs such as Internet traffic introduces high levels of uncertainty and noise~\cite{mao2018variance}. These challenges together make exploration and learning in network environments highly non-trivial, causing several prior work~\cite{pensieveRealWorld, fugu, mao2019park} to conclude that addressing them was essential for high performance in the real world. 

Our analysis discovers a key factor exacerbating the impact of these challenges in practice: the skewed distribution of input traces in training datasets. 
Skew leads to limited exploration of rare or tail-end traces, introduces significant errors in feature learning for these tail-end traces, and results in noisy learning by introducing updates consisting of a narrow set of traces. Moreover, the performance of these tail-end traces is often less than optimal, in contrast to the ``common'' traces which are heavily optimized. Thus, enhancing tail-end trace performance becomes critical for improving the overall controller performance~\cite{datadrivennetworking}. 
Unfortunately, such skew is prevalent in input-driven environments. For example, during an eight-month data collection period, the video streaming platform Puffer~\cite{fugu} recorded that low-bandwidth input traces made up less than $20\%$ of the total traces, with only $4\%$ of these traces having any stalls.
Therefore, addressing skew is essential for (a) avoiding amplifying inherent challenges in learning, and (b) improving the overall performance of the controller across both ``common'' and tail-end input traces.  




While techniques for addressing data skew are prevalent in various contexts~\cite{leevy2021mitigating, dong2021multi, liang2019empirical, zhang2021mimicnet, schaul2015prioritized, bucketper, multiagentper}, standard supervised learning solutions such as oversampling or undersampling specific labeled classes do not apply in the case of reinforcement learning, where the controller learns using states, actions and rewards. The few solutions designed specifically for DRL are inadequate for network controllers because they are restricted to state-level prioritization and fail to capture the trace-centric nature of networking environments (\S~\ref{subsec:PER}). Thus, to effectively address this skew, we introduce a novel approach directly targeting the \textit{input traces} in networking environments.


Input traces, which represent logs of time-dependent complex processes, lack a conventional mechanism to identify and balance the skew with. These traces have no features or labels and do not directly contribute to a loss function. To systematically tackle skew and improve overall controller performance, a mechanism to identify and balance the skew in input-driven environments is needed. To do so, in this work, we introduce a generalizable framework, Plume. Plume employs an automated three-stage process. \textit{Critical Feature Identification:} We automatically determine the critical trace features that enable us to identify the traces. \textit{Clustering:} We employ clustering to convert the critical features into salient identifiers. \textit{Prioritization:} In this stage, we prioritize the clusters, such as to expose the controller to traces where it can learn the most from (\S~\ref{sec:design}).


Using Adaptive Bitrate Streaming, Congestion Control and Load Balancing as representative network applications, we show that Plume offers consistently high performance across a wide range of controllers---across different environments, diverse trace distributions, and multiple DRL algorithms. We also introduce Gelato, a novel DRL ABR controller that, when trained with Plume, offers state-of-the-art performance on the real-world live streaming platform Puffer~\cite{fugu}. It is the first controller on the platform that offers statistically significant improvements in both video quality and stalling. It outperforms the previous state-of-the-art controllers, CausalSim~\cite{alomar2023causalsim} and Fugu~\cite{fugu}, reducing stalls by $75\%$ and $78\%$ respectively. 

To further analyze prioritization strategies, we also introduce a controlled evaluation environment, \textit{TraceBench}. TraceBench is a simplified Adaptive Bitrate environment with parametrically generated traces. Parameterized trace generation enables users to generate a wide range of test trace distributions in a controlled and accurate manner, which can then serve in thoroughly evaluating sampling strategies. 

In summary, we make the following contributions:
\vspace{1mm}
\begin{itemize}[leftmargin=*,nolistsep]
    \item We systematically study an overlooked aspect of DRL training, skewed datasets, and find that they can have a surprisingly large impact on performance.
    \item We propose Plume as a generalizable framework for handling skewed datasets and improving the performance of input-driven DRL controllers. 
    \item We demonstrate the superior performance and robustness of Plume in Congestion Control, Load Balancing, and Adaptive Bitrate Selection, across multiple RL algorithms in simulation and real-world settings. 
    \item We introduce Gelato, a new ABR controller. Plume-trained Gelato, deployed on the real-world Puffer platform~\cite{fugu} for more than a year, is the first controller that achieves signficaint improvements in both video quality and stalling, reducing stalling by as much as $75\%$ over the previous state-of-the-art.
    \item We present a new benchmarking tool, TraceBench, and evaluate the prioritization techniques across a wide range of trace distributions; demonstrating that Plume is robust across them.
\end{itemize}

\noindent We will open-source the code of the Plume library, TraceBench, our DRL environments, and our state-of-the-art ABR controller, Gelato.

\begin{figure*}[t!]
	\centering
    \centering
	\includegraphics[width=0.4\linewidth]{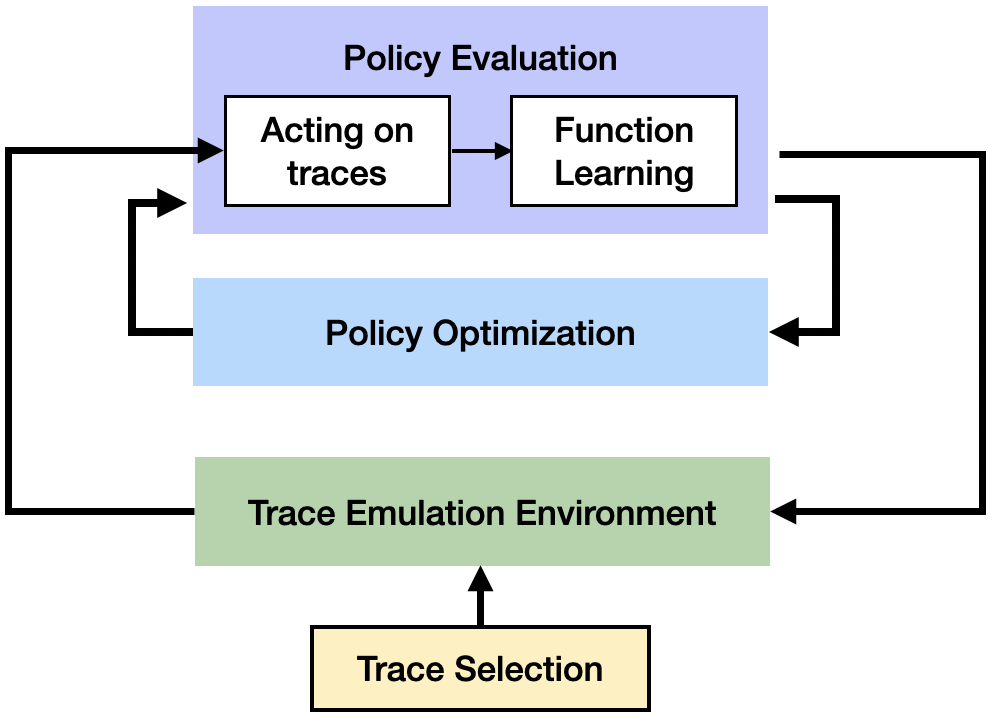}
	\caption{\textbf{RL Training Loop}: The training loop of DRL algorithms in trace-driven environments.}
	\label{fig:train_loop}
\end{figure*}

\section{Background}
In this section, we give a brief overview of reinforcement learning training and our representative applications---adaptive video streaming, congestion control, and load balancing.

\subsection{Reinforcement Learning Preliminiaries}
\label{sebsec:rlproblem}
In Deep Reinforcement learning (DRL), an \textit{agent} interacts with an \textit{environment}. At each timestep, the agent receives the current system state $s_t$, and takes an action $a_t$, drawn from its policy, $\pi (a | s_t)$. The environment plays the action out and \textit{transitions} to the next state $s_{t+1}$, giving the agent a reward $r_t$~\cite{sutton2018reinforcement, SpinningUp2018, silver2015}.

In network environments, non-deterministic network conditions are the primary sources of noise and uncertainty. These conditions determine the environment's response to the controller's chosen actions. For example, in congestion control, external traffic can dictate whether congestion will occur. 

Formally, these conditions are called ``inputs'', and input-driven environments form an Input-Driven Markov Decision Process~\cite{mao2018variance}, defined by the tuple $(S, A, Z, P_s, P_z, r, \gamma)$. Here, $S$ denotes the set of states, $A$ represents the set of actions, $Z$ is the set of training input traces, $P_s$ is the state transition function, $P_z$ is the input transition function, $r$ is the reward function, and $\gamma$ is the discount.

The state transition function $P_s(s_{t+1} | s_t, a_t, z_{t+1})$ defines the probability distribution of the next state $s_{t+1}$ given the current state $s_t$, action $a_t$, and upcoming input value $z_{t+1}$. Meanwhile, the input transition function $P_z(z_{t+1} | z_{t})$ defines the probability of the next input value based on the current one, leading to an effective transition function given by $P_s(s_{t+1} | s_t, a_t, z_{t+1}) P_z(z_{t+1} | z_{t})$.

As depicted in Figure~\ref{fig:train_loop}, the DRL learning process aims to guide the policy $\pi$ towards higher cumulative reward through a loop involving two steps: a \textit{policy evaluation} step and a \textit{policy improvement} step~\cite{horgan2018distributed}. During policy evaluation, the agent assesses its current policy's performance by gathering experience through \textit{acting} in the environment and leveraging this experience in \textit{function learning}. Here, it updates its neural network to learn a form of the value function $v^\pi (s) = \E_{\pi} [G | s_0 = s]$, which is the expected return $G$ starting from state $s$, where $G$ is the discounted sum of rewards $G = \sum_{t=0}^{\infty} \gamma^t r_t$. Subsequently, in the policy improvement phase, the agent modifies policy $\pi$ to maximize $v^\pi$. Through this iterative process of estimating and maximizing the policy's value function, the agent learns in the environment.

\parab{On-policy and Off-policy DRL}. DRL algorithms are broadly divided into two categories based on their policy evaluation stages. \textit{On-policy RL algorithms} perform policy evaluation from scratch in each iteration, using only the data collected with the latest version of the policy for function learning~\cite{sutton2018reinforcement}. These algorithms have found wide application in networking~\cite{pensieve, aurora, mao2016resource}. On the other hand, \textit{Off-policy RL algorithms} continue to use data from older versions of the policy along with new data to improve sample efficiency. They maintain a window of environment transitions, described by the tuple $(s_t, a_t, r_t, s_{t+1})$, in a FIFO buffer known as Experience Replay~\cite{mnih2013playing}. Off-policy algorithms are similarly popular in networking, as used by \cite{orca, xu2023teal}.

\subsection{Environments}
\label{subsec:envionments}
In this paper, we use adaptive bitrate streaming, congestion control, and load balancing as representative networking environments.

\parab{Adaptive Bitrate Streaming}. In HTTP video streaming, the video is divided into short chunks and encoded, in advance, at multiple discrete bitrates. During streaming, the ABR algorithm is responsible for sequentially selecting the bitrate of each chunk to maximize the viewer's Quality of Experience (QoE). While streaming, the client also has a buffer to store chunks yet to be played. Typically, the QoE is measured with a numerical function that awards higher quality, and penalizes both quality fluctuations and rebuffering. The quality of a chunk may be denoted by its bitrate or by more complex measures such as Structural Similarity Index Measure (SSIM)~\cite{ssim}. 

\parab{Congestion Control}. Congestion Control (CC) algorithms are responsible for determining the most suitable transmission rate for data transfer over a shared network. Based on network signals such as round-trip time between the sender and receiver and the loss rate of packets, a CC algorithm estimates sending rate that maximizes throughput and minimizes loss and delay. 

\parab{Load Balancing}. A Load Balancing (LB) algorithm in a distributed cluster decides which server to serve a new job at, such as to minimize the job's total processing time. When a job arrives, the LB algorithm does not know how busy each server is or how long each server will take to process the job. To make a good decision, it uses data such as the time between job arrivals, the duration of past jobs, and the number of jobs already waiting at each server.  

\section{Motivation}
In this section, we discuss the challenges associated with training DRL controllers and how they are exacerbated by skewed training datasets. Then, we give a brief overview of current techniques used to handle skew and motivate the need for prioritized trace sampling. 

\subsection{Challenges with DRL Training}
Input-driven DRL training environments used in networking settings suffer from several overarching challenges.

\parab{Challenge 1: Needle-in-the-haystack exploration}. In input-driven environments, the majority of the state-action space presents little difference in reward feedback~\cite{mao2019park}. In this scenario, standard exploration techniques, which select a random action with $\epsilon$ probability and follow greedy actions otherwise, have a low chance of finding a successful policy. The complexity is further exacerbated by the imbalance in the training datasets, particularly the under-representation of rare or tail-end traces. Such traces are infrequently encountered by the controller, thereby further limiting the opportunity for the controller to discover successful strategies for them. However, performance in these tail-end traces can be crucial for higher overall performance of the controller~\cite{datadrivennetworking}.

\parab{Challenge 2: Noise and Uncertainty}. The network conditions, or inputs, determine the behavior of the environment and constitute the main source of uncertainty. For instance, when an Adaptive Bitrate controller chooses a bitrate, it operates without knowledge of the client's link bandwidth. This unobserved factor directly impacts the amount of time the client will wait for a data chunk. Such variability introduces noise into the learning process, creating a situation where identical states can yield widely different outcomes based on the network conditions~\cite{mao2018variance}. This variability or noise is particularly amplified when the distribution of network traces is skewed. In these cases, a single training iteration may inadequately represent the full spectrum of input traces, thereby leading to divergent or noisy updates. 

\parab{Other Challenges with skew}. Skew in the distribution of input traces presents challenges during the function learning phase of DRL training (Fig.~\ref{fig:train_loop}). Since states are dependent on these input traces, a skewed input distribution leads to a skewed state distribution. This imbalance in the state distribution degrades the performance of the neural network and makes it vulnerable to overfitting~\cite{johnson2020effects, ye2020identifying}.

\subsection{Towards Prioritizing Trace Sampling}
\begin{figure}[t]
    \centering
	\includegraphics[width=0.45\linewidth]{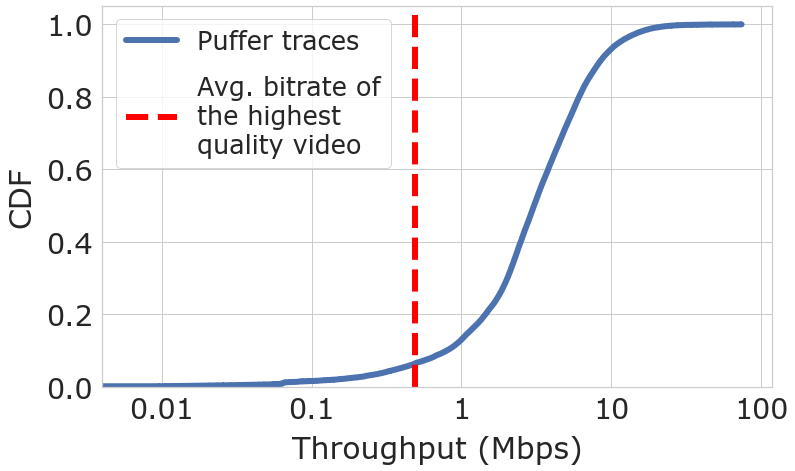}
	\caption{\textbf{Puffer Input Trace distribution}: Distribution of effective throughput of Puffer traces collected during the two-month period, Apr '21 - May '21. Less than $6.5\%$ of traces have average effective throughput below the average bitrate of the highest quality video. Each Puffer stream is a trace.}
    \labeltext{\arabic{figure}}{fig:throughput_dist}
\end{figure}

\label{subsec:PER}

Next, we discuss commonly used ML techniques for handling skew and establish the need for prioritized trace sampling in input-driven environments.

\parab{Prioritized Experience Replay (PER)}.
Off-policy DRL algorithms use a buffer to store past state transitions and apply Prioritized Experience Replay (PER)~\cite{schaul2015prioritized} to sample them during function learning. PER employs prioritization, also known as importance sampling, to prioritize state transitions based on their Temporal Difference error. The key idea is to focus on transitions with higher error, improving the controller's predictions where most needed rather than on the most common transitions. 

While PER is effective in traditional DRL settings~\cite{hessel2018rainbow, horgan2018distributed}, it is limited in addressing the challenges presented by skew in input-driven environments. The reason is that while PER addresses the state skew in the function learning phase, the skew in input traces additionally affects the \textit{acting} phase of the training loop (Fig.~\ref{fig:train_loop}). The controller has limited opportunity to act in tail-end traces. Without modifying which traces are selected during acting phase, PER cannot increase the frequency of exploration in tail-end traces or ensure a comprehensive evaluation across the entire input trace distribution. Consequently, the importance sampling PER uses is not sufficient to handle the skew of input traces in the dataset.

\begin{figure*}[t]
\centering

	\includegraphics[width=0.85\linewidth]{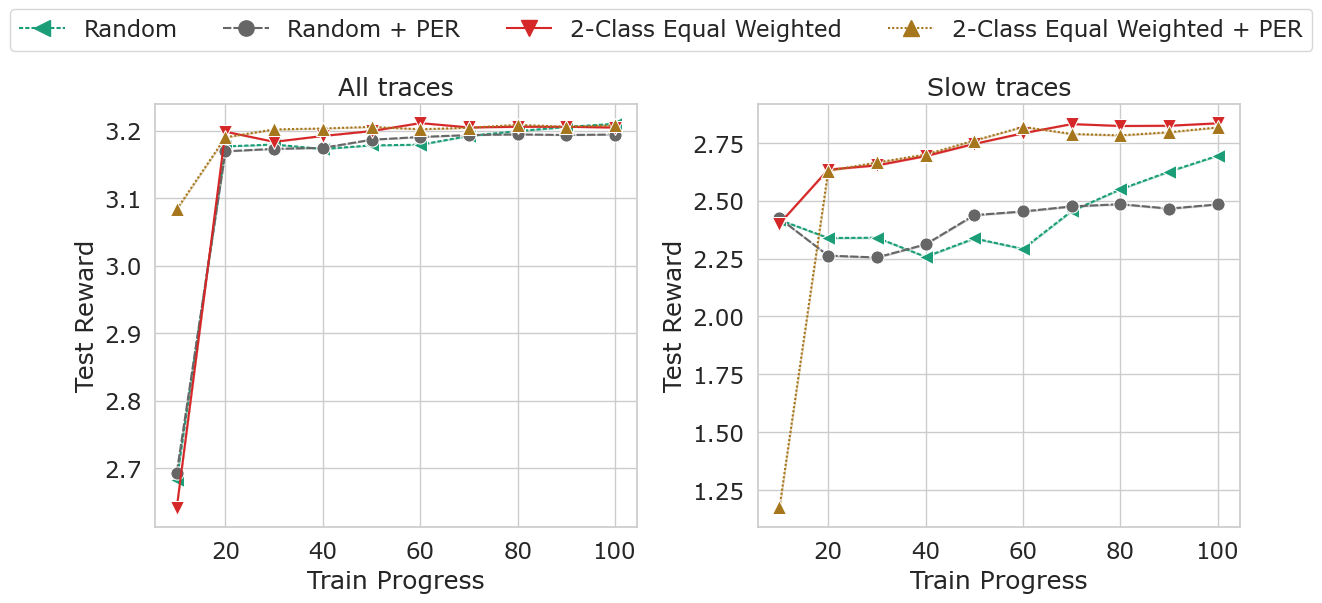}
	\caption{ \textbf{Comparing Prioritization Techniques}: Performance of sampling transitions (PER enabled/disabled) compared with sampling traces (Random vs. 2-Class Equal Weighted Trace Selection) on a DQN-variant of Cannoli ABR controller trained using the \textit{off-policy} algorithm, Ape-X DQN \cite{horgan2018distributed}. 2-Class Equal Weighted Trace Selection offers superior controller performance and training efficiency, while PER does not. $95\%$ confidence interval shown as error bands.}
	\labeltext{\arabic{figure}}{fig:apex_final}

\end{figure*}
\label{subsec:motivation_experiment}
\parab{Prioritized Trace Sampling}. We reexamine the DRL workflow and identify a more suitable location for prioritization. We put forward a simple training paradigm in input-driven environments: prioritizing \textit{trace sampling} during the acting step. With this, we can achieve high state-action space exploration and representative evaluation on all kinds of traces.

To test our hypothesis, we experiment by enabling prioritization at two points in the DRL workflow: sampling transitions in the experience buffer at the function learning step (PER enabled vs. disabled) and sampling input traces in the acting step (Random sampling vs. 2-Class Equal Weighted). 2-Class Equal Weighted is a simple input trace prioritization scheme that divides the traces from the Puffer Platform into two classes, those with mean throughput higher/lower than the highest quality bitrate, 0.98 Mbps (Figure~\ref{fig:throughput_dist}), and equally samples both classes. We evaluate the impact of each technique on a DQN variation of Gelato controller for ABR trained using the state-of-the-art algorithm Ape-X DQN~\cite{horgan2018distributed} (training settings detailed in \S~\ref{sec:gelato} and \S~\ref{subsec:real_world_settings}).

In Figure~\ref{fig:apex_final}, we observe that the simple 2-Class Equal Weighted gives the highest controller performance and training stability. By prioritizing the tail-end slow throughput traces, we achieved high performance in both all and slow network traces without compromising anything. Enabling PER does not significantly improve controller performance. Even though the replay buffer can store $2$ million transitions (over 5000 input traces), the controller performance falls short of the naive trace prioritization scheme. This highlights that the skew in the trace distribution cannot be easily overcome at the function learning step.
\begin{figure*}[t]
	\centering
	\includegraphics[width=\linewidth]{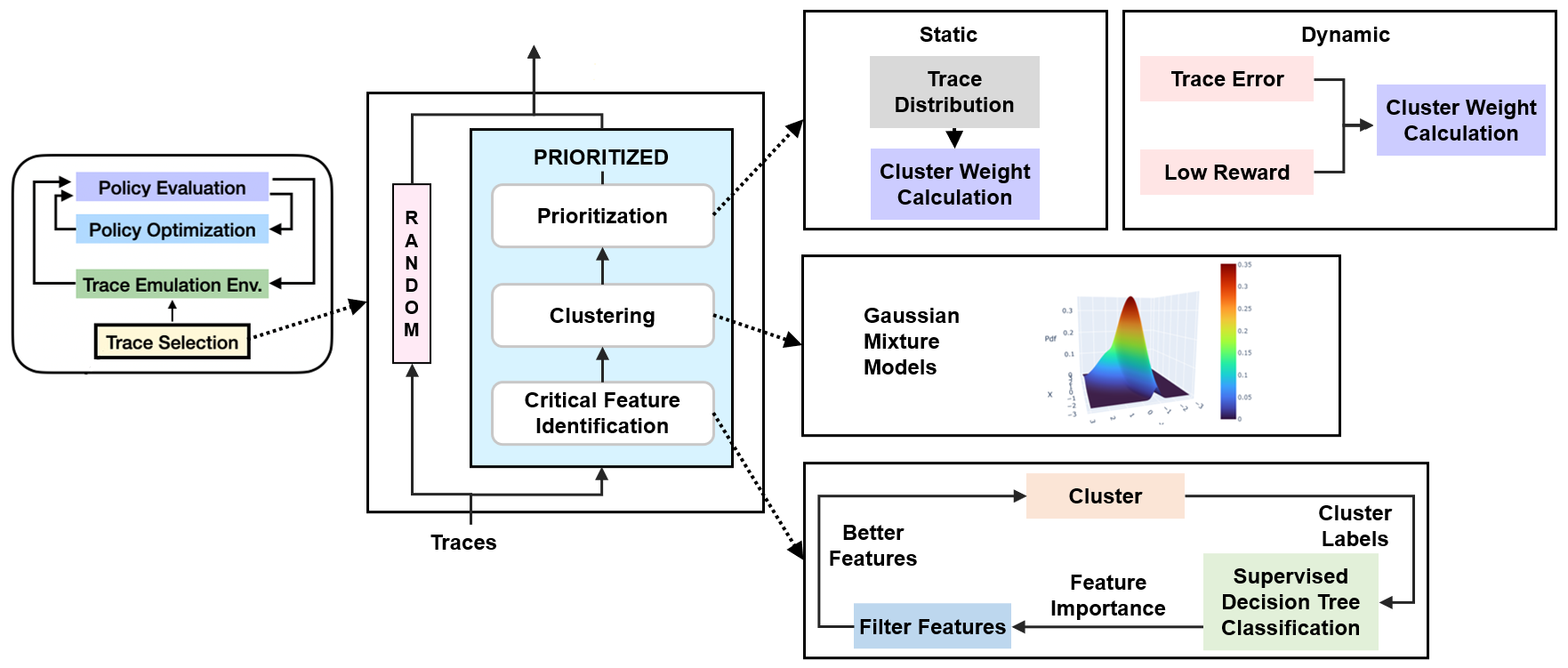}
	\caption{ \textbf{Plume System Diagram}: The Plume Workflow involves three key stages: (1) Critical Feature Identification, where we characterize the traces and their skew, (2) Clustering, where we try to simplify the prioritization problem by grouping traces, (3) Prioritization, where we observe the performance of the agent and attempt to prioritize important trace clusters.}
	\label{fig:trace_selection}
\end{figure*}

\section{Design}
\label{sec:design}
Toward improving the performance of DRL training by balancing skew, we put forward the idea that trace selection is the aptest location for prioritization. 

In order to balance the skew during trace selection, we take advantage of a key observation: input traces inherently correspond to users or workloads, with groups of them sharing similar characteristics.  To ensure a balanced representation of the underlying users, the dataset must contain a roughly uniform number of input traces across them. 
We define input traces to have a set of user attributes $\Phi = [\phi_1, \phi_2, ...]$ given by the function $\Phi = X(trace)$, where $X$ depends on the domain.
These features identify similarities between user traces, and play a key role in obtaining a balanced representation.


Plume is a systematic framework to automatically balance this skew in input traces. Plume allows the agent to have balanced exploration and stable learning updates. Figure~\ref{fig:trace_selection} gives an overview of the Plume workflow. Plume is implemented in the Trace Selection module which is responsible for supplying traces to the simulation environment. This module sits outside of the DRL training loop and is queried by the environment to get traces to replay. Plume has three key stages: critical feature identification, clustering, and prioritization. 

In the critical feature identification stage (\S~\ref{subsec:critical_features}), Plume identifies the attributes of the input traces. In the clustering stage (\S~\ref{subsec:clustering}), it simplifies the prioritization problem by clustering the attributes. Finally, in the Prioritization (\S~\ref{subsec:prioritization}) stage, Plume prioritizes the traces to balance input traces using one of two techniques: static or dynamic.  

\subsection{Critical Feature Identification}
\label{subsec:critical_features}
Input traces, which are time-dependent series of values that define complex external conditions, can be incredibly difficult to characterize and prioritize directly. Hence, the first step towards automated prioritization of traces is identifying the attributes $\Phi$ using critical feature identification.

To extract all features associated with the time series trace data, we rely on the popular feature extraction tool for the time series data, tsfresh~\cite{tsfresh}. We extract a large set of features $[\phi_1, \phi_2, ... \phi_n]$ broadly applicable to all input-driven DRL environments. However, because this large set of features may not be relevant to every application, we introduce an automated three-step process to narrow down to the critical ones, inspired by the idea of recursive feature elimination in supervised learning~\cite{scikit-learn-rfe}.

First, we start with the large set of features and apply clustering to create a fixed small number of clusters. This is denoted by $c = C([\phi_1, \phi_2, ... \phi_n])$, where $c$ is the cluster labels, and $C$ is the clustering function.

Second, we obtain the features most relevant in producing this mapping. To do so, we use the cluster labels $c$ and train decision trees based on the features $[\phi_1, \phi_2, ... \phi_n]$. With this training, we can compute the information gain $IG(c, \phi_i) = H(c) - H(c | \phi_i)$ for each feature $\phi_i$. Here, $H$ is the Shannon entropy of the cluster labels, which is a measure of the average level of ``uncertainty''.

Third, we eliminate features with the lowest $IG$ values. We continue this cycle of clustering, classification, and feature elimination until we are left with only the features that have high information gain. As we eliminate less useful features, we increase the number of clusters to ensure that the final feature set is sufficiently expressive.

Note that the clustering at this stage is solely for feature selection and has no impact on the main clustering phase (\S~\ref{subsec:clustering}).

\subsection{Clustering}
\label{subsec:clustering}
The second stage involves clustering traces using the critical features identified in the previous stage. In this stage, we attempt to reduce the complexity of balancing the skew by obtaining their salient clusters. 

To detect skew in the dataset, we can look at the attributes $\Phi$ of input traces. However, balancing the skew based solely on these attributes proves to be a complex task. This difficulty arises because the attributes, represented as $[\phi_1, \phi_2, \ldots, \phi_n]$, are continuous random variables that may not be independent. In other words, modifying the skew of one attribute could negatively affect the skew of another. To address this issue, we cluster the traces to obtain a single distribution to balance. We achieve this using a clustering algorithm $C$ to obtain the labels $c$ so that the mapping again becomes $c = C(\Phi)$. By doing so, we create a ranking function that allows us to instead prioritize a categorical distribution of input traces, where the cluster labels act as the categories. We represent this distribution as $y$, where $y_i$ is a category, or salient trace cluster within it.

To cluster the traces, we employ Gaussian Mixture Models (GMM) with Kmeans++ initialization~\cite{scikit-learn}. Gaussian Mixture Models use a generalized Expectation Maximization algorithm~\cite{Expectat23:online} and can effectively deal with the large variations found in input data. Thus, GMMs are a good fit for our real-world input trace datasets. However, GMMs can often converge to local optima and require us to know the number of clusters a priori. Hence, to produce an effective clustering automatically, we perform a two-stage search for random initializations used in GMMs and the number of clusters. First, for different cluster counts, we evaluate the GMM's log-likelihood score for the trace features across a range of random initializations and identify the initialization that maximizes the log-likelihood score for each cluster count. Second, we determine the optimal number of clusters from the output of the previous stage based on the highest normalized Silhouette score~\cite{Silhouet31:online}.

\subsection{Prioritization}
\label{subsec:prioritization}
With critical feature identification and clustering stages complete, we have a categorical distribution of input traces $y$ that we can balance by prioritization.

So far, we have discussed balancing the distribution $y$. While this can be done in a number of ways, to ensure that the balancing leads to meaningful performance improvements, we introduce a target function to balance the distribution around: ``reward-to-go''. Reward-to-go represents the additional rewards that a controller can still achieve. This can be formally defined by Equation~\ref{eq:reward_to_go}:
\begin{equation} 
\label{eq:reward_to_go} 
\Delta G_{y_i} = \E_{y_i}[G^{\pi^*} - G^{\pi^\theta}]
\end{equation}
In this equation, $y_i$ is a category (\S~\ref{subsec:clustering}) in the input trace distribution $y$, $G = \sum_{t=0}^{\infty} \gamma^t r_t$ is the discounted return of the trace as described in Section \ref{sebsec:rlproblem}, $G^{\pi^*}$ is the return under the optimal policy $\pi^*$, and $G^{\pi^\theta}$ is the return under the current policy. We aim to balance the input trace distribution based on how suboptimally the current policy performs, ensuring a uniform gap across all traces. In other words, we seek to ensure that target function $\Delta G_{y_i} = \Delta G_{y_j}$ for all categories $y_i$ and $y_j$. However, calculating reward-to-go is often not possible in real-world situations because it depends on variables such as the controller's training parameters and state features, and can require solving an NP hard problem~\cite{meskovic2015optimal}. In this work, we introduce two strategies to approximate this prioritization: Static and Dynamic.

\parab{\textit{Static Prioritization}}. In this approach, we tackle skew by statically balancing the distribution of input traces. Specifically, we adjust the sampling weights to be the inverse of the distribution $y$, as expressed in Equation~\ref{eq:static}:
\begin{equation}
\label{eq:static}
W_{y_i} = \frac{1}{f(y_i)}
\end{equation}
Here, $W_{y_i}$ signifies the prioritization weight for category $y_i$, and $f(y_i)$ is the original probability density function for the categorical distribution $y$. When we sample according to these prioritization weights, we modify the effective probability density function, which now becomes $f'(y_i) = \frac{W_{y_i} f(y_i)}{\sum_{y_k \in y} W_{y_k} f(y_k)}$.

While there exists no analytical way to compute $\Delta G_{y_i}$, in some cases, we can show that static prioritization effectively balances the skew. First, consider that under random trace sampling, the imbalance can be arbitrarily large:

\begin{proposition}
Let \( L \) be a constant and \( y \) be a categorical distribution of input traces. Suppose 
\[
\frac{\Delta G_{y_i}}{\Delta G_{y_j}} \approx \frac{f(y_j)}{f(y_i)},
\]
then there exists a distribution of traces \( y \) such that 
\[
\frac{\Delta G_{y_i}}{\Delta G_{y_j}} \geq L.
\]
\end{proposition}

\begin{proof}
Consider a distribution with two categories where
\[
f(y_1) = \frac{1}{1+L} \quad \text{and} \quad f(y_2) = 1 - f(y_1) = \frac{L}{L+1}.
\]
From the above, it follows that
\[
\frac{\Delta G_{y_2}}{\Delta G_{y_1}} \approx L.
\]
\end{proof}

\noindent However, using static prioritization, this imbalance no longer exists:
\begin{proposition}
Let $y'$ denote the re-weighted categorical distribution of input traces. Suppose
\[
\frac{\Delta G'_{y_i}}{\Delta G'_{y_j}} \approx \frac{f'(y_j)}{f'(y_i)},
\]
then 
\[
\Delta G'_{y_i} \approx \Delta G'_{y_j}.
\]
\end{proposition}
\begin{proof}
From the given condition, we have
\[
\frac{\Delta G'_{y_i}}{\Delta G'_{y_j}} \approx \frac{f'(y_j)}{f'(y_i)} = \frac{W_{y_j} f(y_j)}{W_{y_i} f(y_i)} \cdot \frac{\sum_{y_k \in y} W_{y_k} f(y_k)}{\sum_{y_k \in y} W_{y_k} f(y_k)} = 1.
\]
\end{proof}


Given these propositions, it is evident that under static prioritization, irrespective of the initial input trace distribution, the relative reward-to-go ratio $\frac{\Delta G'{y_i}}{\Delta G'{y_j}}$ is close to one, while this ratio can take arbitrarily large values under random trace sampling.
For both propositions, an underlying assumption is that the ratio $\frac{\Delta G_{y_i}}{\Delta G_{y_j}}$ is approximately equal to the inverse of the ratio of probability densities for the relevant categories. 
This assumption mirrors a real scenario: as the sampling frequency of a category increases, the controller also becomes better at handling traces from that category. Consequently, the reward-to-go gap decreases with increasing sampling probability for the category.

\parab{\textit{Dynamic Prioritization}}. 
In dynamic prioritization, we compute an approximation of reward-to-go that adapts as the training progresses. Reward-to-go of a category can vary as the training progresses, and hence, the extent of prioritization of a category to achieve high performance can differ across categories during training.

\begin{equation}
\label{eq:pts_dynamic}
\begin{aligned}
\Delta G_{y_i} &= \E_{y_i}[G^{\pi^*} - G^{\pi^\theta}]  \\
\Delta G_{y_i} &\approx \E_{y_i}[\hat{G}(\Phi) - G^{\pi^\theta}]  & \text{// Approximate unknown policy} \\
\Delta G_{y_i} &\approx \E_{y_i}[\hat{G}(\Phi) - G^{\pi^{\theta}}] -\E_{y_i}[G^{{\pi}^{\theta}}]&  \text{// Compensate bias in }  \hat{G}(\Phi)
\end{aligned}
\end{equation}



As the optimal return cannot be calculated, we replace the return $G^{\pi^*}$ with 
the expected return of the trace computed based on the corpus of seen traces, $\hat{G}(\Phi)$. Note that $\hat{G}(\Phi)$ is a function approximator that is trained continuously, in parallel with the controller. In a broad sense, this allows us to measure the improvement the controller can still achieve on a given trace. Nonetheless, this approximation is vulnerable to bias, especially in input traces where the controller's performance is poor. In such poor conditions, the estimate may become overly pessimistic and might not accurately capture the reward-to-go. To address this concern, we introduce the second term, $-\E_{y_i}[G^{{\pi}^{\theta}}]$, which gives priority to traces that have low returns. 

The dynamic weights are proportional to the normalized sum of components of $\Delta G_{y_i}$ (Eq.~\ref{eq:pts_dynamic}). Note that our prioritization stage is outside the DRL algorithm's training loop in the Trace Selection module (Fig.~\ref{fig:trace_selection}).



\section{Gelato}
\begin{figure}[t]
	\centering
	\includegraphics[width=0.9\linewidth]{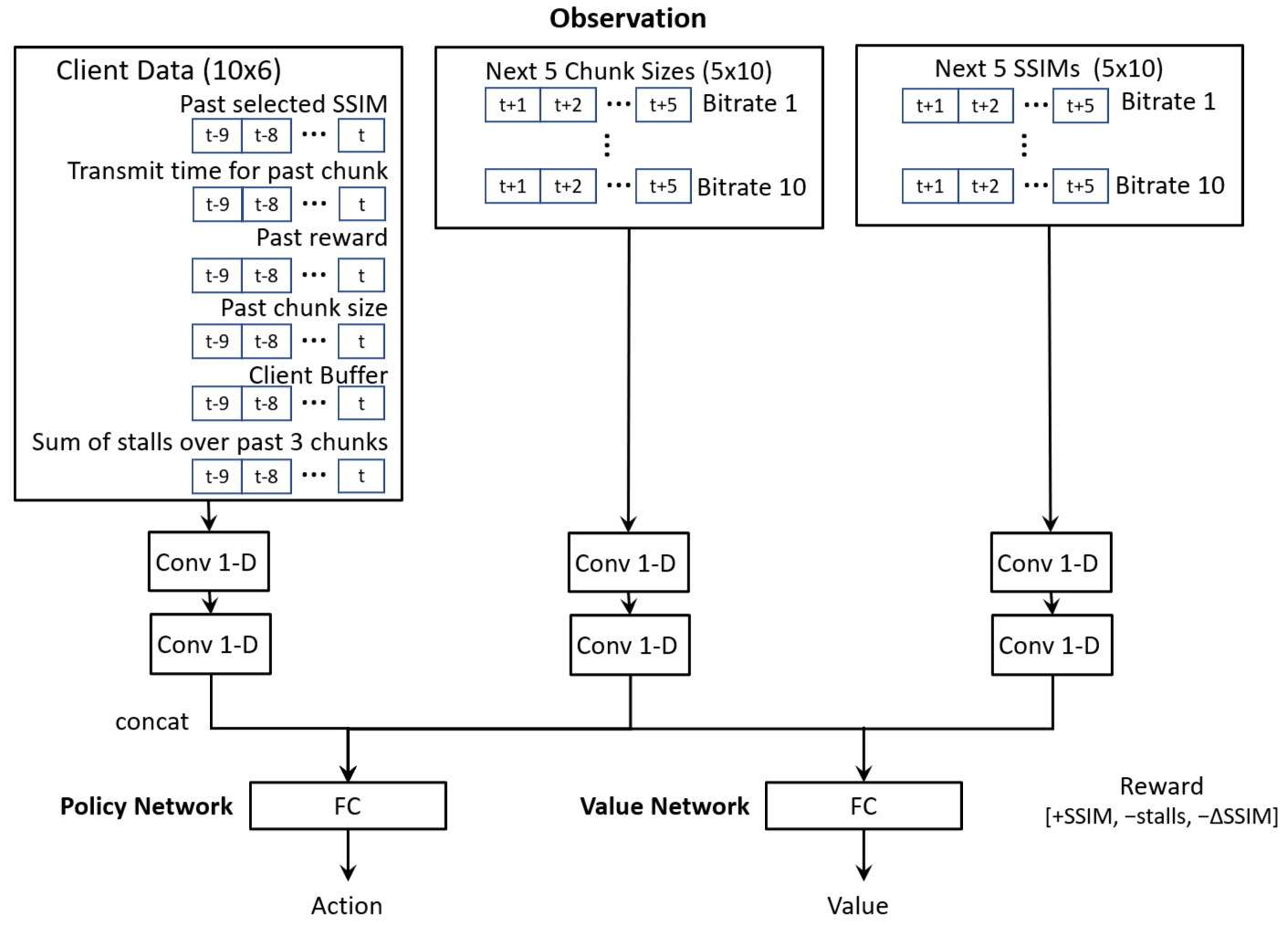}
	\caption{\textbf{Architecture of Gelato}: Gelato takes as input complex features of the video stream.}
	\label{fig:gelato}
\end{figure}

\label{sec:gelato}
We present a new controller architecture for ABR, Gelato. Our design is inspired by Pensieve~\cite{pensieve} but with a number of crucial changes. Unlike other DRL environments, which have simpler controllers, ABR can benefit from a new architecture to improve efficiency and training performance.  We will demonstrate in Section~\ref{sec:experiments} that by combining Gelato and Plume framework we obtain a state-of-the-art controller that surpasses all existing ABR controllers in both real-world and simulated scenarios. For an overview, see Figure~\ref{fig:gelato}.

\parab{Rewards}. We use SSIM as the optimization goal with the reward coefficients used in Fugu~\cite{fugu} ($+$SSIM, $-$stalls, $-\Delta$SSIM). We use the video chunk sizes and SSIM values from the logged data publicly released by the Puffer platform. We also normalize rewards with the transformation $r: sign(r)(\sqrt{|r|+1} - 1) + \epsilon r$ and clipping, where $\epsilon$ is $10^{-2}$. The transformation has been empirically shown to better handle rewards with large scales and varying density~\cite{pohlen2018observe}. It prevents extremely large positive or negative values of reward (e.g for an unusually high SSIM, a long stall, etc) from dramatically affecting the controller.

\parab{Features}. We use rich application-level features, keeping a history of the client buffer and past rewards. We additionally use a history of stalls over a longer $30$ chunk horizon, aggregated using a sum over $3$ chunks. This history allows the controller to get a deep understanding of the client's quality of experience and automatically correct itself when the network conditions become poor. Note that Gelato does not use low-level TCP statistics as Fugu does. However, similar to Fugu, it uses transmit time instead of throughput, and the values of chunk sizes and SSIMs at all encoded bit rates over the next five chunks. These values are often available to ABR controllers because the chunks to be sent are encoded more than $10$ seconds before being sent.  

\parab{Neural Architecture}. We design the deep neural network of Gelato to be efficient. We reduce the total number of parameters by using an additional convolutional layer to downsample the inputs, thereby reducing the input size to the Fully Connected (FC) layer. Gelato's deeper neural network allows for more expressive feature extraction while reducing the number of trainable parameters and Mult-Add operations by $76\%$ and $68\%$ respectively compared to Pensieve. 

For the off-policy DQN variant of Gelato (used for comparison with PER in Figure~\ref{fig:apex_final}), we use the same architecture, swapping the policy and value networks with a single dueling Q-network\cite{wang2016dueling}. For details, see Appendix~\ref{sec:appendix_abr}.

\section{Experiments}
\label{sec:experiments}
In this section, we present the findings of testing the impact of Plume across multiple agent architectures and networking environments, and across simulation and real-world trials.

\subsection{Implementation}
We now turn to detail our implementation of all the experiments performed in this paper. We implement Plume as a Python library compatible with all major DRL frameworks.

\parab{Training environments and algorithms}. We implement the standard ABR environment by extending the Park Project code \cite{mao2019park} and interfacing with Puffer traces~\cite{fugu}. We implement the CC environment by extending the source code provided by Aurora~\cite{aurora}. We implement the Load Balancing environment using the open source Park Project code~\cite{mao2019park}. We use the standard OpenAI Gym~\cite{gym} interface and the RL libraries Stable-Baselines 3~\cite{stable-baselines3} and RLlib~\cite{liang2018rllib}.

\parab{Plume}. We implement Plume completely outside of the DRL workflow in the Trace Selection Module. To implement the critical feature identification stage, we use tsfresh~\cite{tsfresh} for its feature-extraction tools and Scikit-Learn~\cite{scikit-learn} for its decision tree and clustering implementation. To implement the clustering stage, we again use Scikit-Learn for its Gaussian Mixture Model and Silhouette scoring implementation. To implement the prioritization stage, we employ Numpy~\cite{harris2020array} and PyTorch~\cite{paszke2019pytorch}. 

A straightforward implementation of Plume can directly interfere with the various distributed training paradigms used in many DRL algorithms~\cite{mnih2016asynchronous, horgan2018distributed}. To this degree, we implement our prioritization strategy using the distributed shared object-store paradigm in Ray~\cite{moritz2018ray}. This allows us to share the sampling weights across distributed RL processes without interfering with any DRL workflows. 

With our implementation, the overhead for Plume is minimal. The Critical Feature Identification and clustering stages are completed once before training, with runtimes in the order of minutes. In Plume-Dynamic, we train a neural network to map the attributes $\Phi$ of an input trace to the return $G^{\pi^\theta}$ in that trace parallel to the training. We maintain a short bounded history of the trace-return pairs for each category and use this history to compute the two components of our prioritization function. To compute the first term in our approximation, we take the ground-truth samples of trace feature-return pairs, measure the mean absolute error of the neural network for these samples, and average them across each category. To calculate the compensation term, we take the negative of the mean return found in each category. We do this prioritization process continuously, adjusting the weights to the controller's current needs. This dynamic prioritization calculation adds a computational overhead on the order of milliseconds per iteration. This added prioritization computation is handled in parallel to the DRL training and does not slow it down.  

\subsection{Settings}
\label{subsec:real_world_settings}
In this section, we present the settings used in our experiments. We present our results as averages over $4$ instances ($4$ controllers trained using the same scheme with different initial random seeds). This is consistent with the standard reporting practice in the RL community~\cite{horgan2018distributed, r2d2, mnih2016asynchronous}. For testing on the Puffer platform, we select the best of these four seeds for benchmarking. For more details on these settings, see Appendix~\ref{sec:appendix_abr} and \ref{sec:appendix_cc}.

\parab{Adaptive Bitrate Streaming}. For ABR, we use the network traces logged by the Puffer platform over the two-month period of April 2021 - May 2021. The traces are system logs of the video streams. Each trace is a time series of tuples over all chunks sent during the session that includes (i) the chunk sizes and SSIMs at various bitrates, (ii) the bitrate chosen by the ABR algorithm for that chunk, and (iii) the time taken to transmit that chunk. We calculate the effective throughput over time using this data and use it alongside the chunk sizes and SSIMs for simulation. We enforce a minimum trace length requirement of $3$ stream-minutes to reduce I/O overhead. Moreover, during training, we randomly split long traces into lengths of 500 chunks in order to prevent them from dominating training. This results in more than $75,000$ traces, of which we randomly select about $55,000$, representing over 4.25 stream-years, for our analysis. Of these, we use $40,000$ for training and about $15,000$ for testing. We evaluate every controller using the same train and test set. 

\parab{ABR Puffer Platform}. We test Gelato with both random sampling and Plume on the live streaming research platform Puffer from 01 October 2022 - 01 October 2023. The Puffer platform streams live TV channels such as ABC, NBC or CBS over the wide-area Internet to more than $200,000$ users~\cite{puffer-online}. Over this time, we analyzed the ABR algorithms streamed over $58.9$ stream-years of video. 
We report the performance as SSIM vs. stall ratio, following the convention used by the Puffer platform~\cite{fugu}. 

We compare Gelato-Random and Gelato-Plume-Static with the performance of the Buffer-based controller BBA~\cite{bba}, the in-situ continuous training controller Fugu's February version, Fugu-Feb~\cite{fugu}, and CausalSim~\cite{alomar2023causalsim}, a version of Bola~\cite{bola} tuned by trace-driven causal simulation. We note that on the Puffer platform, Gelato-Random is called by its code name ``unagi'', Gelato-Plume-Static is called ``maguro''. We additionally compare Gelato with the original version of Fugu over the period 07 March 2022 - 05 October 2022 in Figure~\ref{fig:puffer_website_plot} in Appendix~\ref{sec:appendix_abr}.


\parab{Congestion Control}. For congestion control, we use the synthetic network traces employed by the DRL CC algorithm Aurora~\cite{aurora}. Here, each trace is represented by $4$ key simulation parameters: throughput, latency, maximum queue size, and loss. For training, we sample throughput from the range $[100, 500]$ packets per second, latency from the range $[50, 300]$ milliseconds, maximum queue size from the range $[2, 50]$ packets, and loss rate from $[0, 2]$ percent. For testing, we broaden the ranges and sample throughput from $[50, 1000]$, latency from $[25, 500]$, maximum queue size from $[2, 75]$, and loss from $[0, 3]$. We sample throughput, maximum queue size, and loss rate spaced evenly in the range on a geometric progression, while sampling latency uniformly evenly. We note that we do this sampling only once and fix it for both training and testing for all controllers.

\parab{Load Balancing}. In load balancing, we use synthetic job traces from the Park Project~\cite{mao2019park}. Each trace represents a time series indicating the size of arriving jobs over time. Following standard parameters, the inter-arrival times are sampled from the exponential distribution $exp(\lambda = 55)$, and the job sizes are sampled from the pareto distribution $pareto(x_m = 1.5, \alpha = 100)$. We limit the trace length to $650$ to ensure that the controller is sufficiently penalized for poor scheduling decisions, and that the variance of returns $G$ remains finite. As in congestion control, we perform this sampling once and fix it for both training and testing.

\subsection{Results}
\begin{figure*}[t]
\centering
\begin{subfigure}{.85\textwidth}
    \centering
    \includegraphics[width=\textwidth]{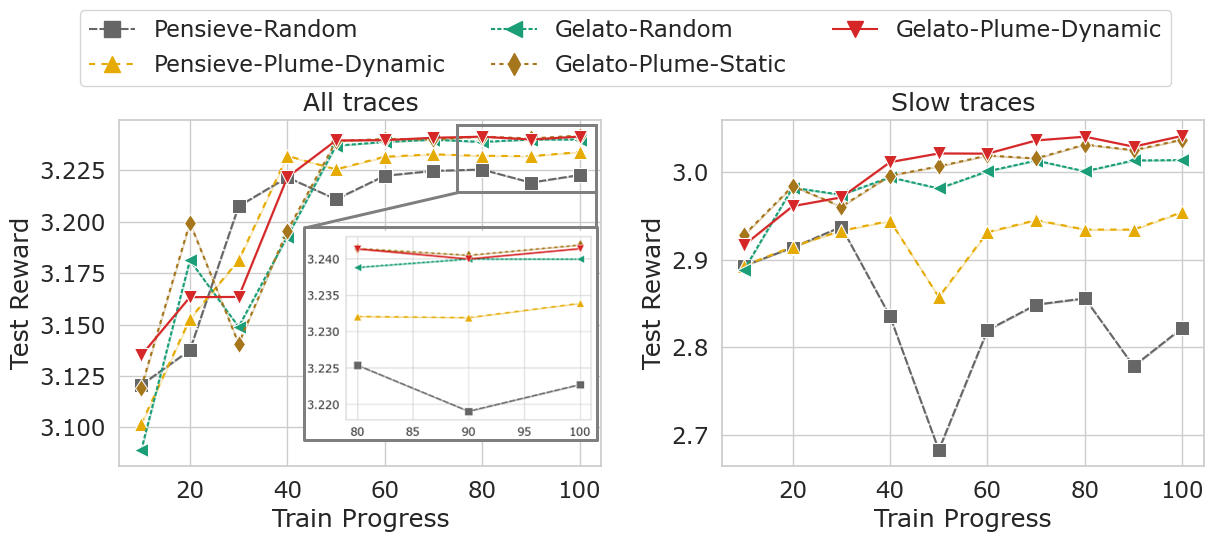}
    \caption{(a) Adaptive Bitrate Streaming}
    \label{fig:abr_sim_all}
\end{subfigure}
\begin{subfigure}{.399\textwidth}
    \centering
    \includegraphics[width=\textwidth]{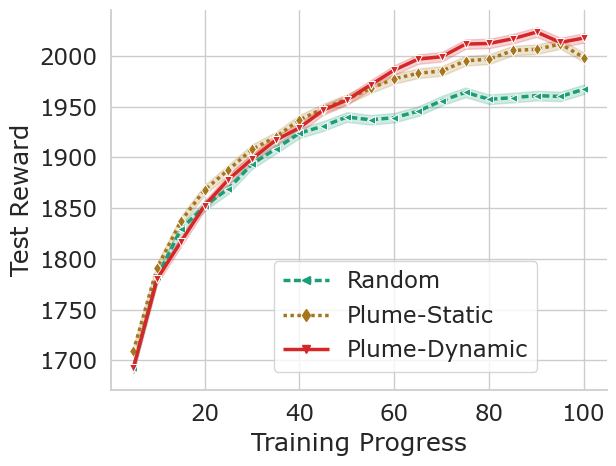}
    \caption{(b) Congestion Control}
    \label{fig:cc_sim}
\end{subfigure}
\begin{subfigure}{.4\textwidth}
    \centering
    \includegraphics[width=\textwidth]{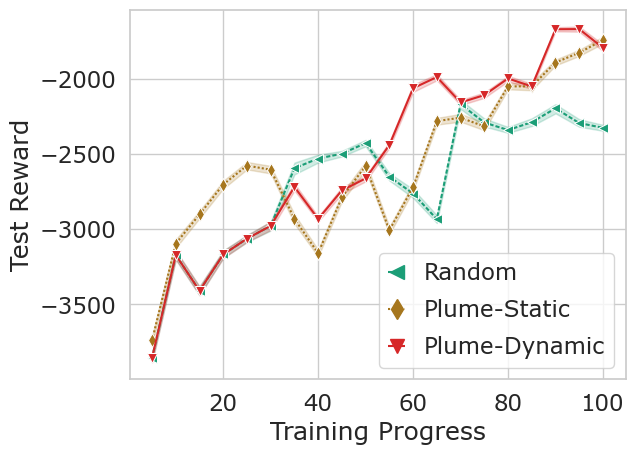}
    \caption{ (c) Load Balancing}
    \label{fig:lb_sim}
\end{subfigure}
\begin{subfigure}[t]{.8\textwidth}
    \centering
    \includegraphics[width=\textwidth]{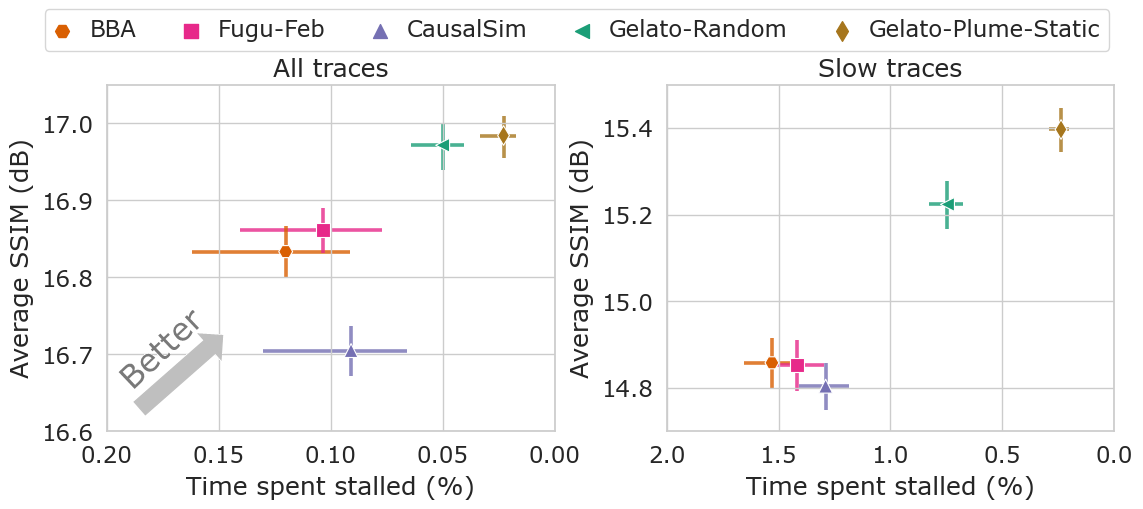}
    \caption{(d) ABR Puffer Platform}
    \label{fig:puffer}
\end{subfigure}

\caption{\textbf{Plume performance over ABR, CC and LB}: Plume outperforms random sampling in simulation and real-world platforms for ABR, CC and LB. The ABR Puffer Platform plots visualize data from the live streaming platform Puffer over the period $01$ Oct '$22$-$01$ Oct '$23$, comprising of over $58.9$ stream-years of video. We re-plot the data from the Puffer website~\cite{puffer-online} to aggregate the different experiment periods together. $95\%$ confidence intervals are shown as error bars and bands. We note that the axis of the plots are different due to inherent differences between the objective functions.}
\label{fig:main_results}
\end{figure*}

In this section, we present the results of our experiments evaluating Plume in ABR, CC and LB. We aim to answer the following questions: How does the performance of Plume compare with random sampling? How do controllers trained with Plume perform in the real world?

In Figure~\ref{fig:main_results}, we present our results comparing Plume with random sampling and other controllers. We present our observations below.

\parab{Plume outperforms random trace sampling across all benchmarks}. With Adaptive Bitrate Streaming, Congestion Control, and Load Balancing, in both simulation and the real world, Plume achieves higher performance than random trace sampling. In Figure~\ref{fig:abr_sim_all}, we analyze the performance of Plume in ABR. We observe that Plume converges to a higher test reward, in both all traces and slow traces. We additionally see that Pensieve-Plume-Dynamic significantly improves upon Pensieve-Random, but that the improvement is not enough to match the performance of Gelato. In Figure~\ref{fig:cc_sim}, we focus on the performance in congestion control. We find a similar trend, with Plume-Dynamic and Plume-Static providing statistically significant improvements in convergence and performance over random input trace sampling. We find the same story in load balancing in Figure~\ref{fig:lb_sim}. We note that while the absolute numerical differences may appear small due inherent scales of the reward function, they exceed the $95\%$ confidence interval bands,  and translate to large real-world differences as we will see next.  

\parab{Plume-Static closely tracks Plume-Dynamic}. 
In Figures~\ref{fig:abr_sim_all}, \ref{fig:cc_sim} and \ref{fig:lb_sim}, we observe that Plume-Static, which employed a simpler prioritization strategy, closely tracks the performance of Plume-Dynamic. This is likely due to the fact that in these scenarios, the impact of shifting reward-to-go values or difficult input traces is minimal. However, as we will see later in Section~\ref{sec:pts_benchmark}, when the training distribution is anomalous or is significantly different from the testing distribution, Plume-Dynamic can prove effective over Plume-Static.  


\parab{Gelato outperforms state-of-the-art controllers in the real world streaming live television over a 1-year period}. To further understand the benefit of Plume, we run Gelato with Plume-Static and random sampling on the real-world live-streaming Puffer platform~\cite{fugu}. We opted for Gelato combined with Plume-Static for this evaluation given its analogous performance to Plume-Dynamic in ABR, but with a simpler design. Additionally, we included Gelato with random sampling as a baseline for comparative analysis. In Figure~\ref{fig:puffer}, we see that Gelato-Plume-Static outperforms the current state-of-the-art controllers Fugu-Feb and CausalSim, alongside the heuristic-based BBA in both SSIM and stalling. Although prior work~\cite{fugu, alomar2023causalsim} reported statistically significant stalling improvements on Puffer, Gelato distinguishes itself by becoming the first ABR controller to achieve statistically significant improvements in both quality and stall reduction. This is particularly noteworthy as Gelato does not depend on low-level TCP metrics like Fugu or intricate simulation techniques that CausalSim uses.

Over this 1-year period, the algorithms we analyze streamed over $58.9$ stream-years of videos to over $200,000$ viewers across the Internet~\cite{puffer-online}. Over this duration, Gelato-Plume-Static achieves $75\%$, $78\%$ and $81\%$ stall reduction compared to CausalSim, Fugu and BBA respectively (Fig.~\ref{fig:puffer}). Gelato-Plume-Static additionally achieves SSIM improvements of $0.28$, $0.12$ and $0.15$ dB over CausalSim, Fugu and BBA respectively. Note that this quality improvement over BBA is more than $5\times$ that of Fugu, which only managed a $0.03$ dB improvement over BBA. CausalSim did not provide an SSIM improvement over BBA over this period. Gelato-Plume-Static has an average SSIM variation of $0.77$ dB, compared to $0.67$, $0.53$ and $0.78$ dB of CausalSim, Fugu and BBA respectively. Moreover, we find that Gelato-Random is a strong baseline, achieving $0.27$ dB SSIM improvement and $45\%$ stall reduction over CausalSim. We make similar observations when comparing Gelato with the original version of Fugu in Figure~\ref{fig:puffer_website_plot} in Appendix~\ref{sec:appendix_abr}.



\section{Plume Benchmarking}
\label{sec:benchmarking}
\label{sec:pts_benchmark}

Having established the performance of Plume on real-world controllers and experiments in Section~\ref{sec:experiments}, in this section, we thoroughly microbenchmark Plume in order to study its impact in isolation. We demonstrate Plume's ability to offer high performance and robustness across various trace distributions. 

\subsection{Settings}
To evaluate Plume and various prioritization strategies, we introduce a controlled ABR environment for microbenchmarking prioritization: TraceBench. TraceBench makes two key changes to the standard ABR environment. First, it simplifies the quality-of-experience measurement to include only two terms, quality and stalling. Second, it parameterizes the traces by key characteristics of real-world traces: mean and variance of network throughput. With these modifications, we can reliably and thoroughly evaluate the controller across various network conditions. 
While this design is a simplification of the real-world environment, it gives a good approximation of a wide range of realistic settings. We believe that the development of such a framework is important for the community as real-world datasets do not allow microbenchmarking of DRL controllers or prioritization strategies. We envision that it can be used by other controller designers. We note that parameterized trace generation is a part of TraceBench, used to create a variety of scenarios to evaluate controllers on. It is not part of any of the prioritization strategies. 

In generating Traces for TraceBench, we focus on traces with two levels of mean throughput, slow and fast, and two levels of variance of the throughput, high variance and low variance. In this benchmark, we generate three sets of datasets with different proportions of these traces: Majority Fast, Balanced, and Majority Slow. See Figure~\ref{fig:toy_abr_dataset} for a visualization of the trace distributions, and see Figure~\ref{fig:toy_abr_trace_viz} in Appendix~\ref{sec:appendix_tiny_abr} for a visualization of example traces. Note that the training and testing datasets remain disjoint: a controller trained on the Majority Fast dataset does not have access to the testing version of it.

We use the off-policy RL algorithm Ape-X DQN~\cite{horgan2018distributed}. To evaluate prioritization in isolation, we use the same DRL hyper-parameters for all agents and present results averaged over $4$ instances. For details of our training parameters, see Appendix~\ref{sec:appendix_tiny_abr}.

\begin{figure*}[t]
	\centering
	\includegraphics[width=0.5\linewidth]{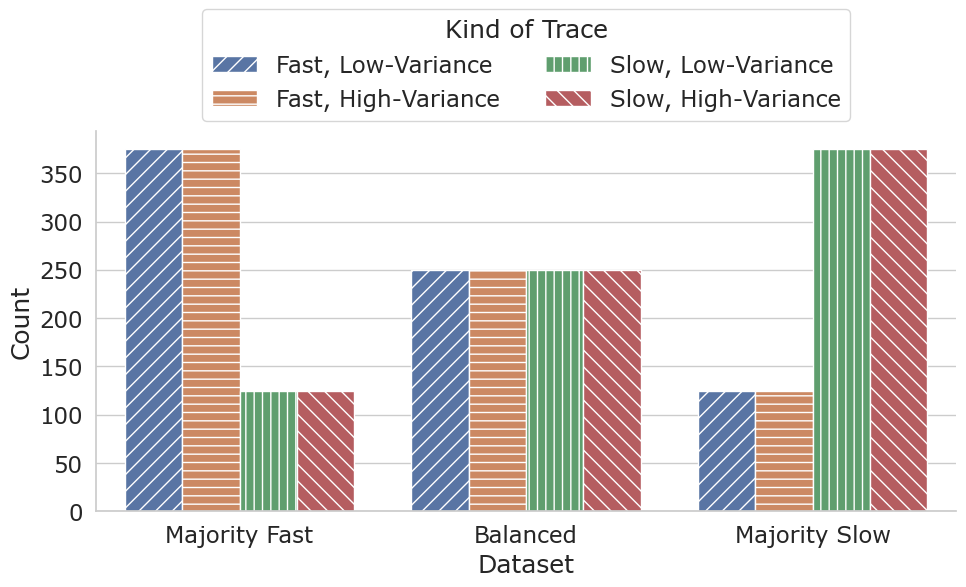}
	\caption{\textbf{Trace Datasets of TraceBench}: Distributions of traces present in each dataset employed in TraceBench. The broad range of trace distributions allows us to thoroughly benchmark prioritization techniques.}
    \labeltext{\arabic{figure}}{fig:toy_abr_dataset}
\end{figure*}

\subsection{Results}
\begin{figure}[t]

	\centering
	\includegraphics[width=0.9\linewidth]{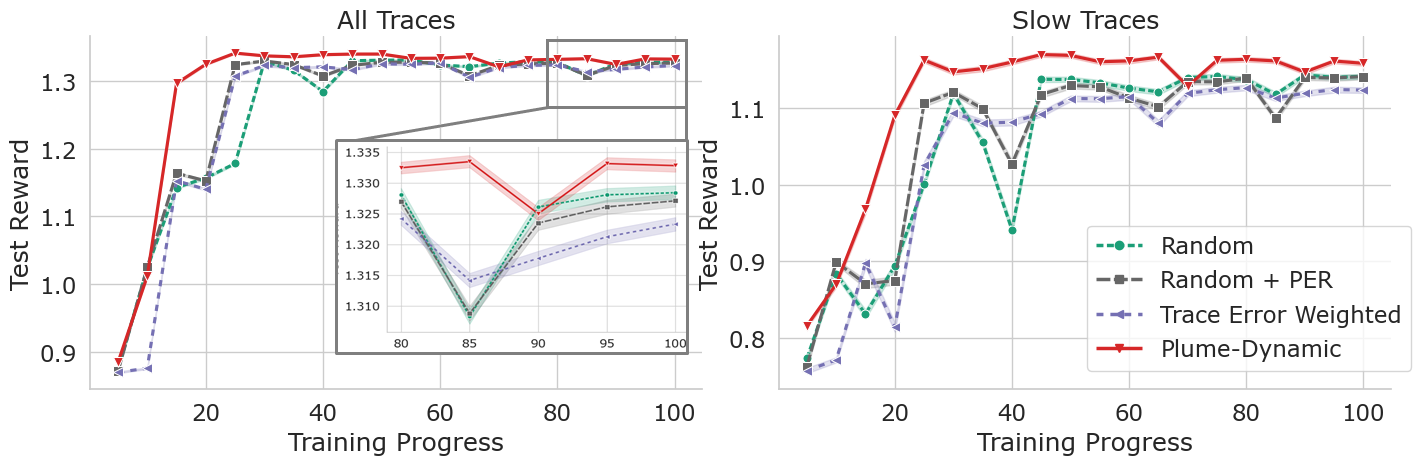}
	\caption{\textbf{ Comparing Prioritization Techniques}: Performance of random sampling, Prioritized Experience Replay (PER), Trace-Error Weighted sampling and Plume-Dynamic on the Majority Fast dataset of TraceBench. Plume-Dynamic, which balances both Trace-Error and Low-reward weights, offers the highest performance. $95\%$ confidence interval shown as error bands.}
    \labeltext{\arabic{figure}}{fig:toy_abr_error_weights}
 \vspace{0.425cm}
\end{figure}

\begin{figure*}[t]
    \centering
\begin{subfigure}[t]{.9\textwidth}
    \centering
    \includegraphics[width=\textwidth]{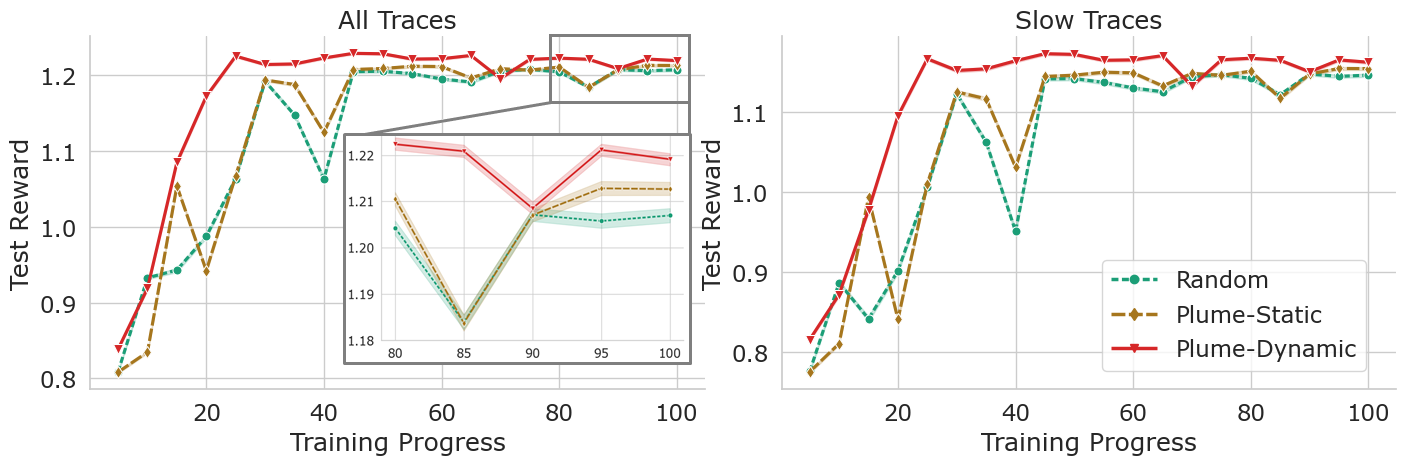}
    \caption{(a) Scenario 1: Training on the Majority Fast, Testing on the Majority Slow dataset.}
    \label{fig:toy_abr_fast_to_slow}
\end{subfigure}

\begin{subfigure}[t]{.9\textwidth}
    \centering
    \includegraphics[width=\textwidth]{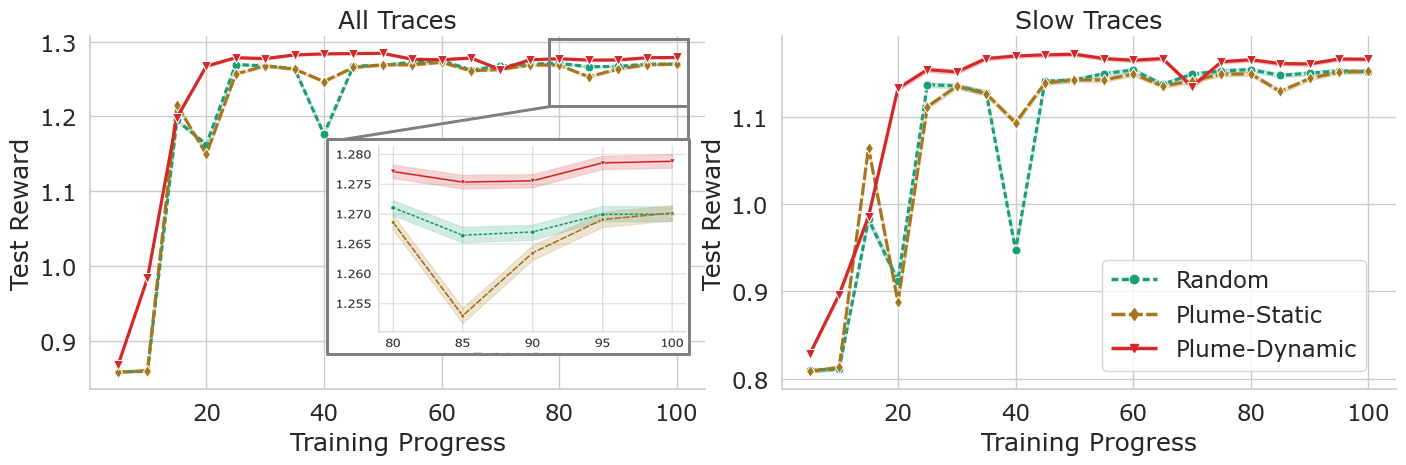}
    \caption{(b) Scenario 2: Training and Testing on the Balanced dataset.}
    \label{fig:toy_abr_balanced}
\end{subfigure}

\begin{subfigure}[t]{.9\textwidth}
    \includegraphics[width=\textwidth]{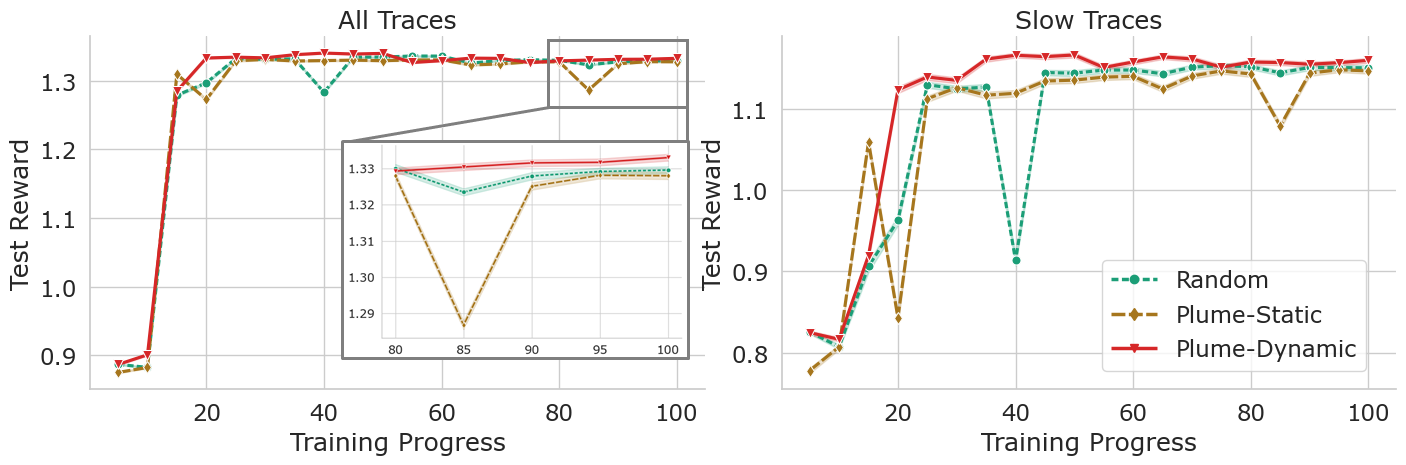}
    \caption{(c) Scenario 3: Training on the Majority Slow, Testing on the Majority Fast dataset.}
    \labeltext{\arabic{figure}}{fig:toy_abr_slow_to_fast}
    \label{fig:toy_abr_slow_to_fast}
\end{subfigure}

\caption{ \textbf{Benchmarking Plume across Trace Distributions}: We benchmark prioritization techniques across different training and testing trace distributions. Plume-Dynamic provides generalizable performance improvement, beating random sampling and Plume-Static in scenarios (1), (2) and (3). $95\%$ confidence interval shown as error bands.}
\label{fig:toy_abr_results}
\end{figure*}

\begin{figure}[t]
	\centering
	\includegraphics[width=0.85\linewidth]{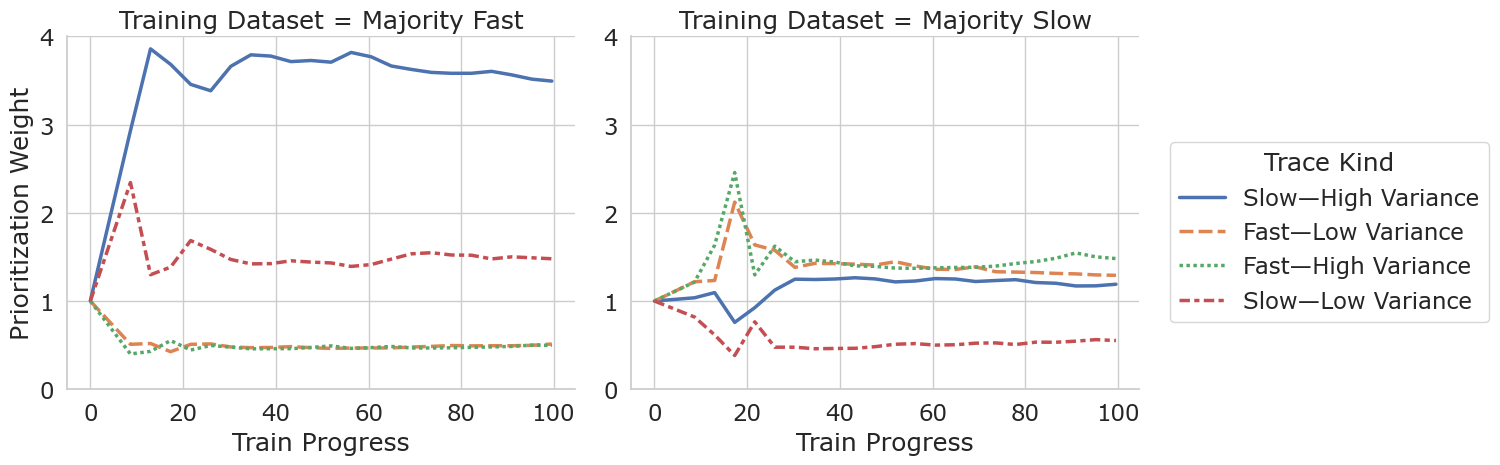}
	\caption{\textbf{Visualization of the prioritization found by Plume-Dynamic in various datasets}: The relative change in sampling weight for each kind of traces over the training progress. Selecting all kinds of traces at weight $1$ is equivalent to random sampling. We note that the ground-truth labels (e.g Slow-High Variance) are not provided to Plume. }
	\label{fig:toy_abr_pts}
\end{figure}

Our experiments investigate two important questions.
First, we evaluate how the versions of Plume's prioritization, Plume-Static and Plume-Dynamic, compare to random trace sampling, the standard trace sampling technique. We additionally evaluate the impact of PER~\cite{schaul2015prioritized}, the state prioritization technique described in Section~\ref{subsec:PER}. Second, we investigate how sensitive these methods are to network conditions distribution shifts.
We would like to emphasize here that these experiments are possible in real-world settings.

\parab{Focusing on tail-ended performance is important}. We start our evaluation with an ablation study on the impact of the approximation terms in Plume-Dyanmic, as presented in Sec.~\ref{subsec:prioritization}. We evaluate a version of Plume-Dynamic without the compensation term. We refer to this version as ``Trace-Error''. We use Majority Fast dataset for this evaluation as it models real-world workloads closely. In Figure~\ref{fig:toy_abr_error_weights}, we present the results of random trace sampling, PER, Plume-Dynamic, and Trace-error on the Majority Fast dataset. We observe that Trace-error can be worse than random sampling, particularly in slow traces, where it not only converges to a low reward but has high variance over the training interval. Meanwhile, Plume-Dynamic, which balances both Trace-error and Low-reward, offers significantly better performance in both Fast and Slow traces. This result highlights the fact that focusing on the low performing traces can be vital to generalizable performance.

In Figure~\ref{fig:toy_abr_results}, we analyze the performance of Plume across various training and testing trace distributions. Particularly, we analyze the following scenarios: 
\begin{itemize}[leftmargin=*,nolistsep]
    \item Scenario 1: The training distribution is similar to the real world but the testing is adversarially different, i.e., we train on the Majority Fast but test on the Majority Slow dataset.
    \item Scenario 2: Both training and testing have a balanced set of traces, i.e., we train and test on the Balanced dataset.
    \item Scenario 3: The training distribution largely consists of the tail end of the testing distribution, i.e., we train on the Majority Slow but test on the Majority Fast dataset.
\end{itemize}

\parab{Plume outperforms random sampling regardless of trace distribution}. As we observe in Figures \ref{fig:toy_abr_fast_to_slow} and \ref{fig:toy_abr_balanced} for the test reward for scenarios (1) and (2), Plume-Dynamic provides a significant performance improvement over random sampling. Moreover, even in Figure~\ref{fig:toy_abr_slow_to_fast} for scenario (3), where we may least expect prioritization to help, Plume-Dynamic is still better than random sampling. We additionally observe that Plume-Static, which performs well in scenario (1), falls behind Plume-Dynamic in scenarios (2) and (3) where the training input trace distributions are either anomalous or are dramatically different from the testing distribution. 

\parab{Plume-Dynamic effectively adapts to all training trace distribution}. To better understand how Plume-Dynamic so effectively generalizes across all of these trace distributions, we visualize the sampling weight of different traces during training in Figure~\ref{fig:toy_abr_pts}. We observe that while training on the Majority Fast dataset, it undersamples the Fast traces and oversamples the Slow ones. In the Majority Slow dataset, it undersamples the Slow--Low Variance traces while oversampling the Fast and Slow--High Variance ones. This highlights the power of Plume-Dynamic's automated prioritization: It adapts itself to the distribution in each dataset and allows the controller to focus on clusters with the most to learn from.

\parab{Controllers trained with Plume are robust to trace distribution shifts}. In the second row of plots in Figures~\ref{fig:toy_abr_fast_to_slow}, \ref{fig:toy_abr_balanced} and \ref{fig:toy_abr_slow_to_fast}, we visualize the Slow-Traces performance of different prioritization schemes. We observe that random trace sampling's performance in slow traces is largely dependent on its training dataset. If the training dataset had few Slow traces, as in scenario (1), the performance is significantly worse than it is in scenario (3), where it had many. However, Plume-Dynamic's performance is robust to the training trace distribution: the controllers all converge to a similar reward in all three scenarios. In the ever-changing landscape of users, devices, and infrastructure inherent to the network domain, this added robustness can be particularly important. It reduces the need for retraining and ultimately reduces the compute requirements and energy consumption of the entire system.


Below, we summarize the findings of our experiments with ABR, CC and LB presented in Section~\ref{sec:experiments}, and the analysis of our extensive Plume benchmarking presented in this section.
\begin{itemize}[leftmargin=*,nolistsep]
    \item Plume is a generalized solution for DRL training in input-driven environments that automatically balances the trace distribution, and offers significant improvement in performance over random sampling in ABR, CC and LB, in simulation and in real-world testing, over both on-policy and off-policy algorithms.
    \item Plume's prioritization strategies work across trace distributions, providing controllers with greater performance and robustness in all.
    \item Gelato trained with Plume offers the best performance when compared to prior ABR controllers on the real-world Puffer platform. It achieves $75\%$ and $78\%$ reduction in stalls over CausalSim~\cite{alomar2023causalsim} and Fugu~\cite{fugu} respectively. It also achieves a statistically significant SSIM improvement of $0.28$ dB over CausalSim and $0.12$ dB over Fugu.
\end{itemize}
\section{Discussion and Limitations}
\label{sec:defense}
We envision Plume to open a new avenue of research in the context of DRL training. Rather than evolve into another hyperparameter that needs tuning in complex RL settings, the problem of trace sampling lends itself well to principled analysis, and in turn a generalized and broadly applicable solution. However, our work still leaves a gap for future work to build upon.

\parab{The need for systematic study of input-driven DRL training}. Our analysis of Plume highlights the significant impact of skew and the benefits derived from addressing it. This finding provides a strong motivation to explore other overlooked factors that may also influence input-driven DRL training. While the broader ML community has conducted in-depth studies on training parameters~\cite{andrychowicz2020matters}, DRL environments~\cite{clary2019let}, and evaluation metrics~\cite{agarwal2021deep}, there is a lack of such research in the networking domain. Engaging in systematic studies in this front could enable the research community to better understand the potential of existing solutions and pave way for an empirical assessment of the real challenges faced by optimized input-driven DRL solutions.

\parab{Future direction for Plume}. In addition to networking environments, Plume can also be beneficial in other trace-driven DRL settings such as drone control, autonomous driving, etc. Plume, as we presented it, cannot be used directly in such environments with more complex input processes. However, extensions to Plume as presented in this paper may be an interesting future direction.

\parab{Sim2Real Gap}. Plume changes \textit{which} traces get sampled and not \textit{how} they are simulated. Plume does not address the problem related to the gap between the simulation environment and the real-world setting (Sim2Real Gap). Solutions that bring simulation closer to reality while still maintaining training efficiency can be combined with Plume. 



\parab{Large-Scale Training}. It is possible that the benefits of higher state-action exploration and feature learning offered by Plume diminish with a very deep neural network over a large number of training steps and parallel environments. Our experimental evidence suggests that Plume is highly relevant for practical DRL environments and training settings. However, we cannot ascertain the effectiveness of Plume at the scale of state-of-the-art Go agents~\cite{muzero},
which requires training capabilities only available to large companies.
\section{Related Work}

\parab{Prioritization in Supervised learning}. Class imbalance is frequently a challenge in supervised data-driven networking problems, where samples of some classes of network conditions or scenarios occur rarely~\cite{leevy2021mitigating, dong2021multi, liang2019empirical, zhang2021mimicnet}. A popular technique to address this problem is to oversample or undersample certain classes to ensure that the model does not drown out the error in the minority classes~\cite{kaur2019systematic}. Such techniques cannot be used in reinforcement learning, where the learning happens using states, actions and rewards rather than a fixed dataset with labels.

\parab{Prioritization in DRL}. While we present the first systematic methodology of prioritization of \textit{input traces} in DRL, prioritization/importance sampling has been applied at other points in the DRL workflow. PER~\cite{schaul2015prioritized} is used to prioritize transitions in the replay buffer in actor-critic algorithms~\cite{acer}, in the multi-agent setting~\cite{multiagentper}, and in text-based DRL environments~\cite{bucketper} to improve sample efficiency. Horgan et.al~\cite{horgan2018distributed} used PER in conjunction with distributed \textit{acting} to improve feature learning. Schulman et.al~\cite{trpo, ppo} employed importance sampling to reduce variance of on-policy training. However, as shown in our experiment (\S~\ref{subsec:motivation_experiment}), these prior solutions do not address the skew in input-driven environments.


\parab{DRL for Networking and Systems applications}. While Plume focuses on improving training over a given input dataset, Gilad et.al.~\cite{gilad2019robustifying} employed RL to find additional training traces that can help the DRL agent generalize to unseen network conditions. Building on this idea, Xia \ea~\cite{xia2022genet} introduced a systematic Curriculum Learning based approach to generate additional environment configurations. Both of the techniques generate \textit{additional} training material for performance in unseen conditions. However, while both of these techniques have been shown to improve the performance of controllers trained with limited datasets (of a few hundred traces), generative solutions have not been demonstrated to provide competitive performance on real-world platforms such as Puffer to the best of our knowledge. In contrast, Plume-trained Gelato continues to be the best performing controller on Puffer since October 2022. 
We note that Plume's training exclusively uses publicly available datasets and does not require solving any trace generation problems.
Mao et.al.~\cite{mao2018variance} introduced the algorithm-side optimization of using input-dependent baselines to reduce the variance of on-policy algorithms at the policy optimization step. Since Plume works outside the DRL training loop, it can also be used in conjunction with any such algorithm-side optimizations. Doubly Robust estimation~\cite{datadrivennetworking} helps in estimating performance variations during input-driven evaluation but does not address the skew in the dataset directly.


\section{Conclusion}
Practical adoption of DRL-based network controllers is limited because the research community does not fully know how to produce high performant controllers. We uncover that skew in the input datasets of DRL controllers plays a significant role in performance, and  put forward Plume, a systematic methodology for addressing skew in input-driven DRL. We thoroughly study the impact of Plume, and show that Plume provides generalizable performance improvement across multiple trace distributions, DRL environments and algorithms. Our novel DRL-based ABR controller, Gelato, trained with Plume offers state-of-the-art performance on the real-world live streaming platform Puffer over more than a year. Plume opens a new avenue of research for methodical control over DRL training in input-driven networking environments and beyond.

\clearpage
\bibliographystyle{ACM-Reference-Format}
\bibliography{reference}
\clearpage

\clearpage
\appendix

\begin{table*}[!h]
\centering
\begin{tabular}{p{6.5cm}p{6.5cm}}
    \toprule
    \textbf{Potential Trace Features} & \textbf{Parameters for feature} \\
    \midrule
    Mean &  -- \\
    Quantile &  $2.5^{th}$\\
    Quantile &  $5^{th}$\\
    Quantile &  $95^{th}$\\
    Truncated Mean & $5^{th}$ quantile \\
    Truncated Mean & $12.5^{th}$ quantile \\
    Truncated Mean & $25^{th}$ quantile \\
    Absolute Fourier Transform Spectrum & Spectral Centroid \\
    \hline
    Ratio of Values beyond standard dev. & Beyond $1\times$ standard dev. \\
    Ratio of Values beyond standard dev. & Beyond $2.5\times$ standard dev. \\
    Variation Coefficient &  -- \\
    Central Approximation of Second Derivative & Mean Aggregation \\
    Truncated Mean Absolute Change & Truncated beyond $5^{th}$ and $95^{th}$ quantile\\
    Truncated Mean Absolute Change & Truncated beyond $1.25^{th}$ and $98.75^{th}$ quantile\\
    Autocorelation & Lag of 3 \\
    Autocorelation & Lag of 4 \\
    Autocorelation & Lag of 8 \\
    \bottomrule
\end{tabular}
\caption{All of the Trace features extracted using the library tsfresh~\cite{tsfresh}. These features are extracted for each trace dataset and then automatically filtered by our novel feature selection technique.}
\label{tab:pts_critical_param}
\end{table*}

\begin{figure*}[t]
    \centering
\begin{subfigure}{.9\textwidth}
    \centering
    \includegraphics[width=\textwidth]{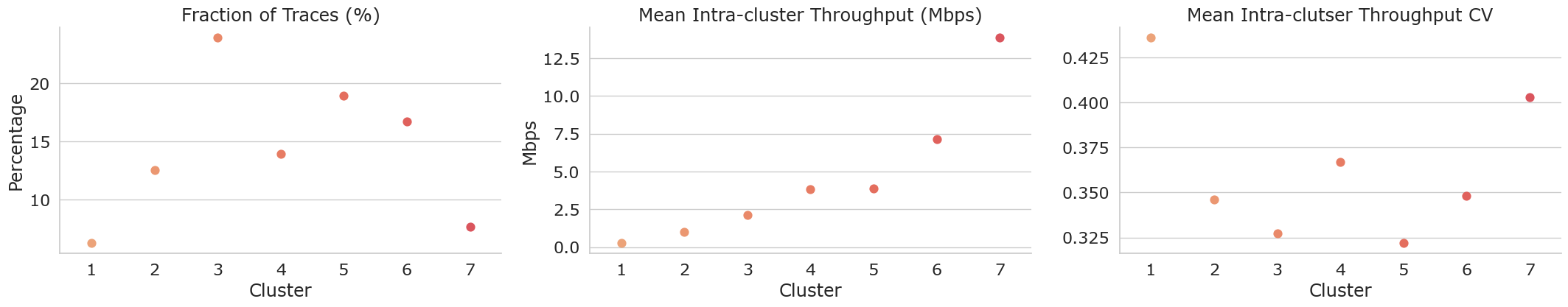}
    \caption{(a) Adaptive Bitrate Streaming Traces}
\end{subfigure}
\begin{subfigure}{.9\textwidth}
    \centering
    \includegraphics[width=\textwidth]{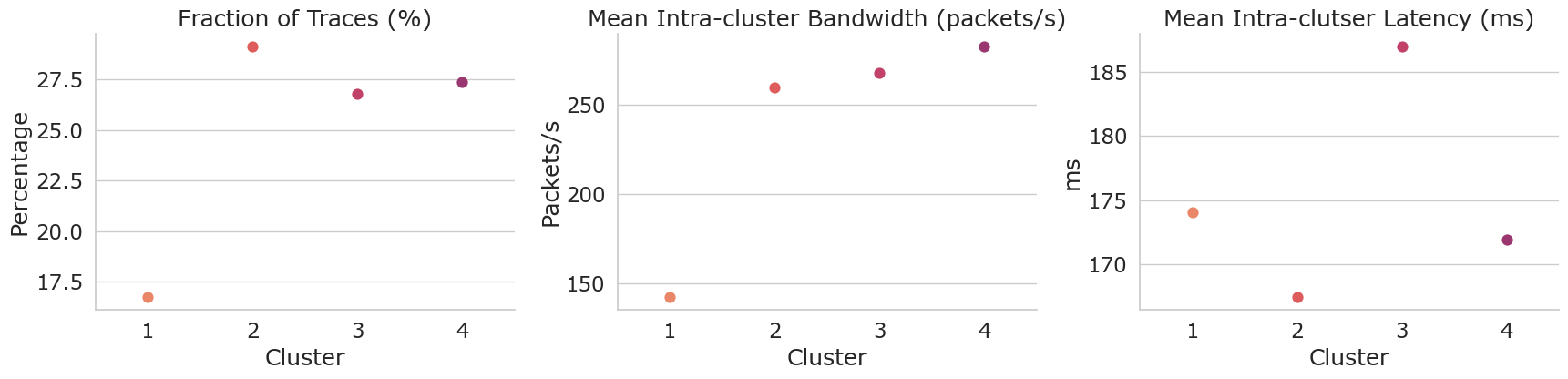}
    \caption{(b) Congestion Control Traces}
    \vspace{4mm}
\end{subfigure}
\begin{subfigure}{.9\textwidth}
    \centering
    \includegraphics[width=\textwidth]{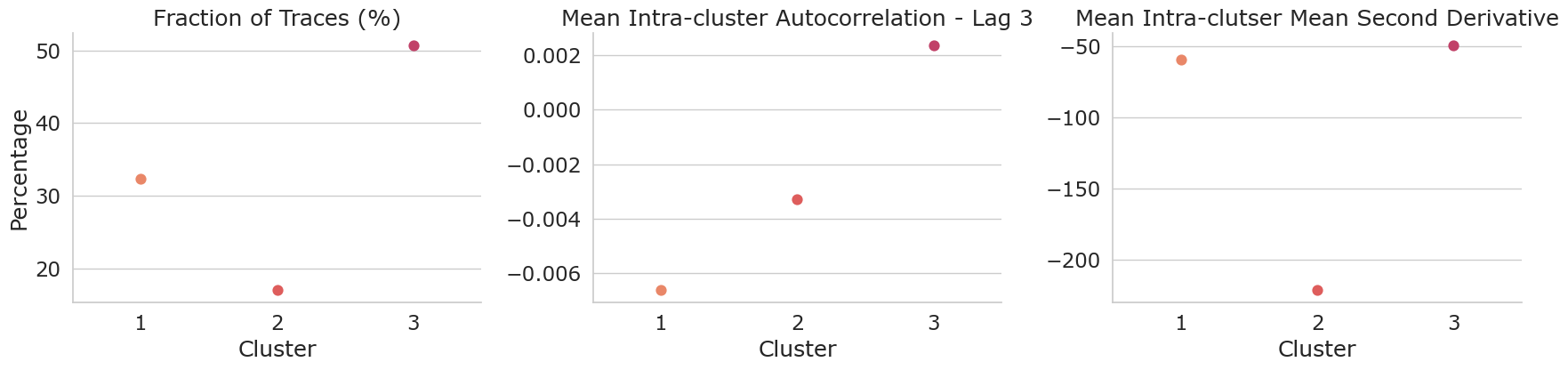}
    \caption{(c)  Load Balancing Traces}
    \vspace{4mm}
\end{subfigure}
\caption{ \textbf{Visualization of Plume Clustering in ABR, CC and LB}: We visualize the clustering automatically produced by Plume in ABR, CC and LB. We see that Plume produces minimal clusters while also separating salient characteristics such as mean throughput and latency. Note that throughput CV is the per-trace coefficient of variation of throughput, and that the axes are different.}
\label{fig:pts_clustering}
\end{figure*}

\begin{figure*}[t]
    \centering
\begin{subfigure}{.75\textwidth}
    \centering
    \includegraphics[width=\textwidth]{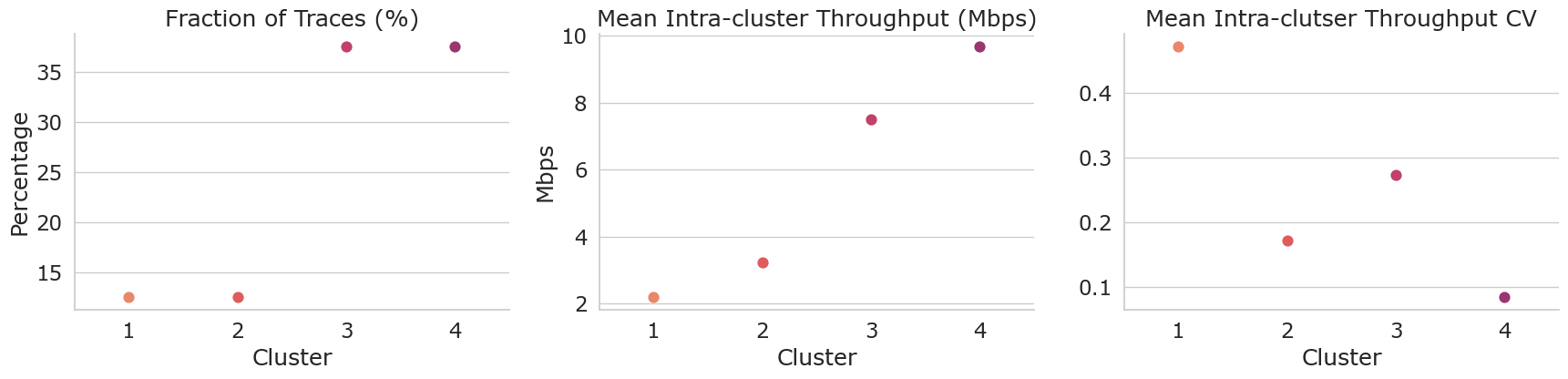}
    \caption{(a)  Majority Fast Dataset}
\end{subfigure}
\begin{subfigure}{.75\textwidth}
    \centering
    \includegraphics[width=\textwidth]{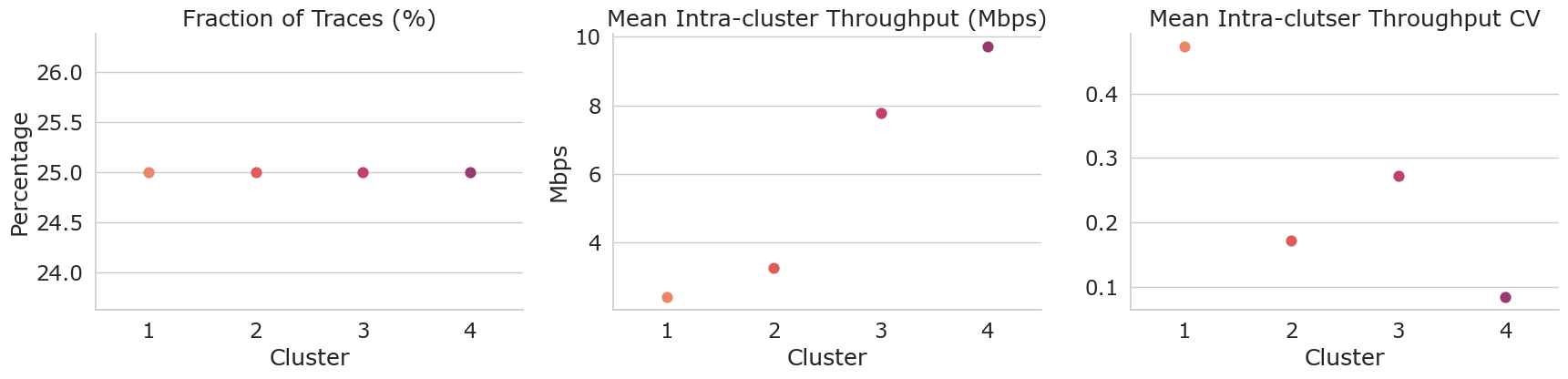}
    \caption{(b)  Balanced Dataset}
\end{subfigure}
\begin{subfigure}{.75\textwidth}
    \centering
    \includegraphics[width=\textwidth]{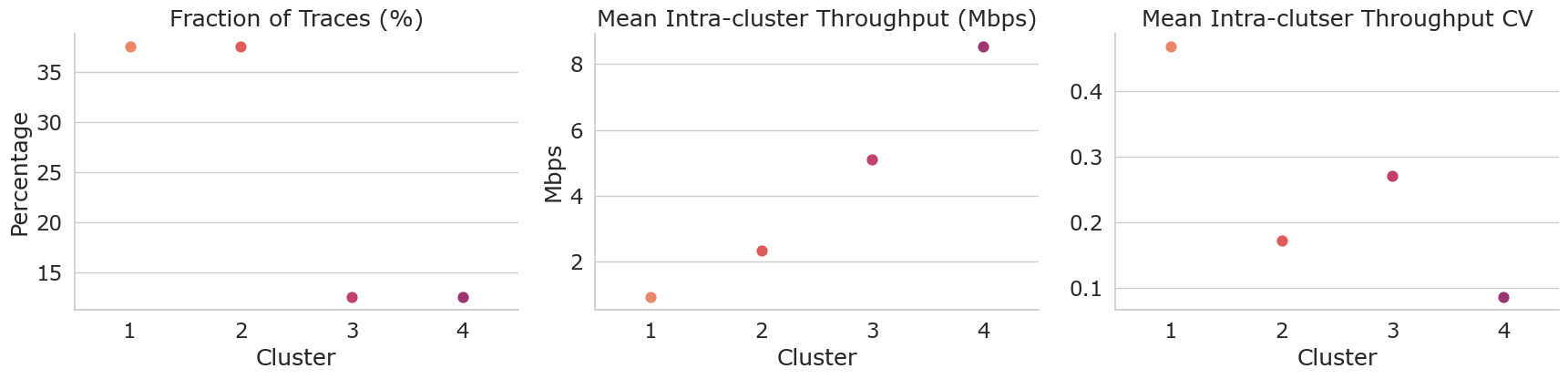}
    \caption{(c)  Majority Slow Dataset}
\end{subfigure}
\caption{ \textbf{Visualization of Plume Clustering in TraceBench}: We visualize the clustering automatically produced by Plume in the Majority Fast, Balanced, and Majority Slow datasets. In each of the datasets, we see that Plume can successfully separate the two levels of throughput and variance. Note that throughput CV is the per-trace coefficient of variation of throughput, and that the axes are different across all plots.}
\label{fig:pts_tiny_abr_clustering}
\end{figure*}

\begin{figure*}[t]
	\centering
	\includegraphics[width=0.4\linewidth]{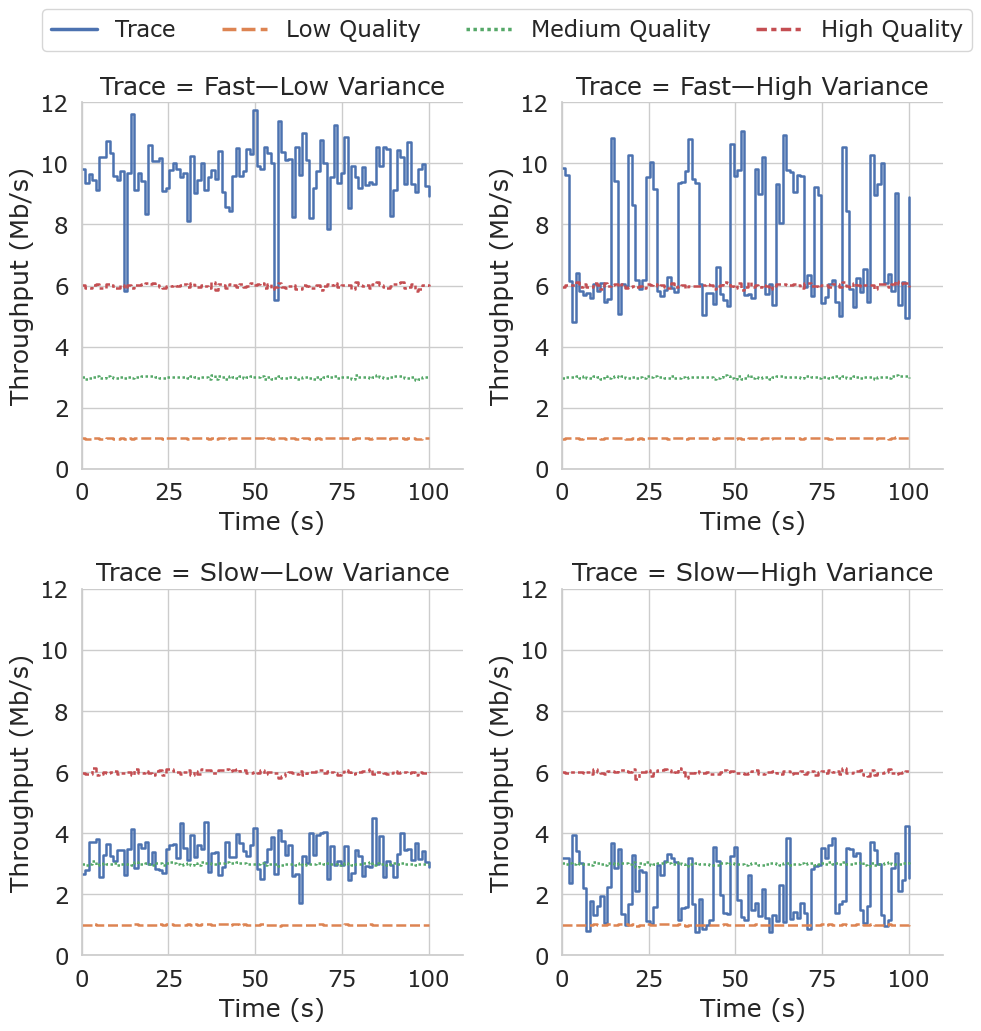}
	\caption{ \textbf{Visualization of Traces generated in TraceBench}: A Throughput vs Time plot of example traces used in TraceBench. The broad coverage of the mean and variance of the throughput requires the agent to learn to adapt to each kind of trace differently.}
	\label{fig:toy_abr_trace_viz}
\end{figure*}

\clearpage

\section{Prioritized Trace Sampling Details}
\label{sec:appendix_pts}
In this section, we provide details, visualizations and analysis of the Plume and its three stages.

\subsection{Critical Feature Identification}
We recall that in the Critical Feature Identification stage, Plume identifies traces by first extracting a wide range of features and then filtering them to find the critical features.

A wide range of features is extracted for each trace in the dataset of traces. Then, this set of features goes through our automated filtering process. During this process, about $40\%$ of the features are eliminated. In Table~\ref{tab:pts_critical_param}, we present the list of all the features extracted. The list contains 16 features, of which 7 describe the central tendency and 9 describe the spread of the input values. 

In TraceBench, the following critical features of the network throughput traces are identified. \textit{Majority Fast dataset}: Truncated Mean Absolute Change of $5^{th}$ and $95^{th}$ quantile, Truncated Mean Absolute Change of $1.25^{th}$ and $98.75^{th}$ quantile, Truncated Mean of the $5^{th}$ quantile, Truncated Mean of the $12.5^{th}$ quantile, Truncated Mean of the $25^{th}$ quantile, and Variation Coefficient. \textit{Balanced dataset}: Truncated Mean Absolute Change of $5^{th}$ and $95^{th}$ quantile, Truncated Mean Absolute Change of $1.25^{th}$ and $98.75^{th}$ quantile, Truncated Mean of the $5^{th}$ quantile, Truncated Mean of the $12.5^{th}$ quantile, Truncated Mean of the $25^{th}$ quantile, and Variation Coefficient. \textit{Majority Slow dataset}: Autocorrelation with lag 3, Autocorrelation with lag 8, Truncated Mean of the $5^{th}$ quantile, Truncated Mean of the $12.5^{th}$ quantile, Truncated Mean of the $25^{th}$ quantile, and Variation Coefficient. 

In ABR, the following critical features of throughput are identified: Autocorrelation with lag 3, Mean, Spectral Centroid of the Absolute Fourier Transform Spectrum, $2.5^{th}$ quantile, $5^{th}$ quantile, $95^{th}$ quantile, Ratio of values beyond $2.5\times$ standard deviation, Truncated Mean Absolute Change of $5^{th}$ and $95^{th}$ quantile, and Truncated Mean Absolute Change of $1.25^{th}$ and $98.75^{th}$ quantile. 

In Congestion Control, because traces are not time-varying series, but instead a tuple of key simulation values, the tuple is treated as the set of critical features. These key simulation values include Bandwidth, Latency, Max. Queue Size, and Loss.

In Load Balancing, the following features of the incoming job sizes over time are identified as key: Autocorrelation with lag 3, Truncated Mean Absolute Change of $5^{th}$ and $95^{th}$ quantile, Spectral Centroid of the Absolute Fourier Transform Spectrum, Mean, Central approximation of Second Derivative, $5^{th}$ Quantile Truncated Mean, $12.5^{th}$ Quantile Truncated Mean, and Variation Coefficient.

We observe that Plume finds different features to be critical for different datasets. This highlights the ability of Plume to effectively adapt to the distribution of training traces to successfully separate them.

\subsection{Clustering}
We recall that in the Clustering stage of Plume, we group similar traces together to attempt to reduce the complexity of the prioritization problem from a trace-level to a cluster. 

We do this by automatically finding both the clustering and the optimal number of features through a search procedure. In TraceBench, we search for the number of clusters in the range [3, 7]. In ABR, we search in the range [6, 15], in CC, we search in the range [4, 9] and in the range [3, 8] in LB. In Figures~\ref{fig:pts_clustering} and \ref{fig:pts_tiny_abr_clustering}, we visualize the clustering found by Plume. We see that Plume effectively groups and separates traces in all six trace datasets.



\section{TraceBench Details}

\label{sec:appendix_tiny_abr}
In designing TraceBench, our objective is to create an environment to thoroughly evaluate and validate different prioritization techniques.

We build our environment on top of the standard ABR implementation found in the Park Project~\cite{mao2019park}. We allow the client to have a maximum buffer of 15 seconds. We consider traces with a maximum length of 100 seconds, with chunks of 1 second. The chunk sizes are generated by sampling a Gaussian distribution around the bitrates [1.0, 3.0, 6.0] megabytes per second.

When generating the traces, we consider two levels of throughput, fast and slow, and two levels of variance, high-variance and low-variance. When generating a trace, we use a 2-state Markov model switching between high and low throughput with different switching probabilities for each kind of trace. In Figure~\ref{fig:toy_abr_trace_viz}, we present a throughput vs. time visualization of each of the four different kinds of traces.

When training the controllers in TraceBench, we use the state-of-the-art feed-forward DQN algorithm Ape-X Dqn~\cite{horgan2018distributed}. We use framestacking of history length 10. We additionally use a standard reward normalization function~\cite{pohlen2018observe} to normalize the rewards. We use the training parameters defined in Table~\ref{tab:tiny_abr_train_param}. We use a simple fully connected architecture with 2 layers of 256 units. We additionally use the dueling and double DQN architecture with a hidden fully connected layer of 256 units.

\section{Adaptive Bitrate Streaming Details}
\label{sec:appendix_abr}
\begin{figure*}[t]
	\centering

        \includegraphics[width=0.95\linewidth]{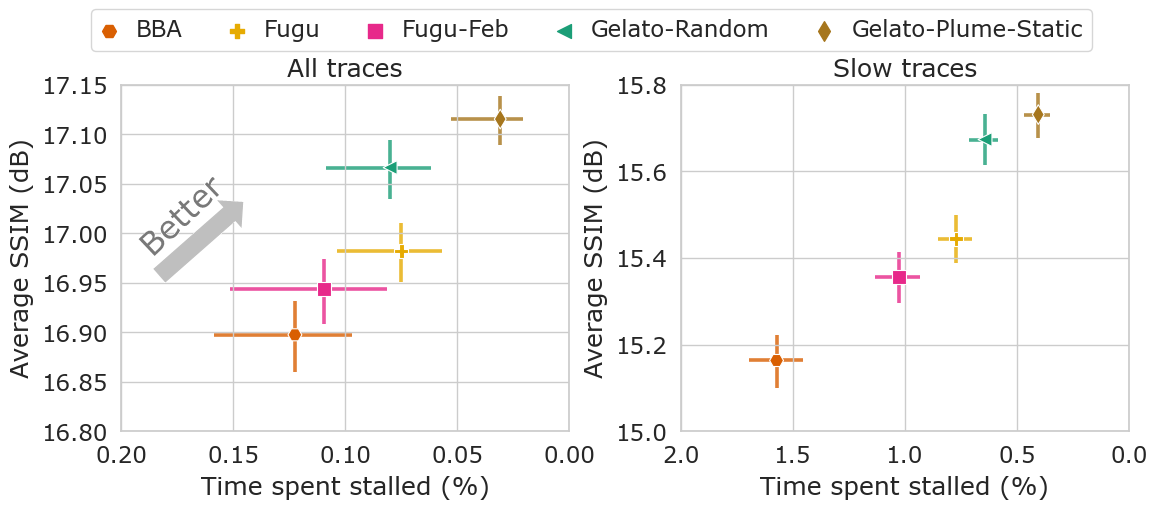}

	\caption{ Performance Plots from the Puffer Platform~\cite{puffer-online}, presenting results from 07 Mar' 2022---05 Oct' 2022. The results visualize 25.5 steam-years of data. Similarly to our main results, we see that Gelato-Plume-Static (maguro) outperforms all other state-of-the-art ABR controllers in both video quality and stalling, and that Gelato-Random (unagi) improves overall video quality while achieving similar stalling performance.}
	\label{fig:puffer_website_plot}
\end{figure*}

\begin{table}[t]
\centering
\begin{tabular}{p{4.8cm}p{2.4cm}}
    \toprule
    \textbf{Hyperparameter} & \textbf{Value} \\
    \midrule
    Learning rate & 0.001\\
    Number of parallel envs. & 64\\
    Number of training steps & $4\mathrm{e}{8}$\\
    Update horizon ($t_{max}$) & 15 env. steps \\
    GAE N-step return & 15 \\
    $\gamma$ & 0.95 \\
    Value function coefficient in loss & 0.9 \\
    Entropy & [5.75, 0.0025] \\
    Entropy annealing interval & $2\mathrm{e}{8}$ steps \\
    Max Gradient Norm & 0.4 \\
    \bottomrule
\end{tabular}
\caption{ Gelato's training hyperparameters. Parameters left unspecified follow the default ones provided in Stable-Baselines3 v2.0~\cite{stable-baselines3} for the A2C algorithm.}
\label{tab:gelato_params}
\end{table}

\begin{table}[t]
\centering
\begin{tabular}{p{4.8cm}p{2.4cm}}
    \toprule
    \textbf{Hyperparameter} & \textbf{Value} \\
    \midrule
    Learning rate & 0.000125\\
    Number of parallel envs. & 16\\
    Number of training steps & $5\mathrm{e}{6}$ \\
    Update horizon ($t_{max}$) & 15 env. steps \\
    GAE N-step return & 15 \\
    $\gamma$ & 0.975 \\
    Value function coefficient in loss & 0.05 \\
    Entropy & [0.1, 0.005] \\
    Entropy annealing interval & $2.5\mathrm{e}{6}$ steps \\
    Max Gradient Norm & 0.25 \\
    \bottomrule
\end{tabular}
\caption{ Aurora's training hyperparameters. Parameters left unspecified follow the default ones provided in Stable-Baselines3 v2.0~\cite{stable-baselines3} for the A2C algorithm.}
\label{tab:aurora_params}
\end{table}

\begin{table}[t]
\centering
\begin{tabular}{p{4.8cm}p{2.4cm}}
    \toprule
    \textbf{Hyperparameter} & \textbf{Value} \\
    \midrule
    Learning rate & $2\mathrm{e}{-4}$\\
    Number of parallel envs. & 16\\
    Number of training steps & $5\mathrm{e}{6}$ \\
    Batch Size & 256 \\
    GAE $\lambda$ & 0.975 \\
    Advantage Normalization & None \\
    N epochs per update & 30 \\
    Value function coefficient in loss & $1\mathrm{e}{-4}$ \\
    Entropy & [0.1, $1\mathrm{e}{-6}$] \\
    Entropy annealing interval & $5\mathrm{e}{6}$ steps \\
    Clip Range & 0.1 \\
    Max Gradient Norm & 0.2 \\
    \bottomrule
\end{tabular}
\caption{ Load Balancing training hyperparameters. Parameters left unspecified follow the default ones provided in Stable-Baselines3 v2.0~\cite{stable-baselines3} for the PPO algorithm.}
\label{tab:lb_params}
\end{table}

\begin{table}[t]
\centering
\begin{tabular}{p{4.8cm}p{2.4cm}}
    \toprule
    \textbf{Hyperparameter} & \textbf{Value} \\
    \midrule
    Number of actors & 4\\
    Number of training steps & $4\mathrm{e}{6}$ \\
    Learning rate & $7.5\mathrm{e}{-6}$\\
    Replay batch size & 32 \\
    $\gamma$ & 0.975 \\
    Replay buffer size & 250000 \\
    N-step return & 7\\
    $\epsilon$ annealing interval & $7\mathrm{e}{5}$ steps\\
    Value clipping & [-32, 32] \\
    \bottomrule
\end{tabular}
\caption{ TraceBench training parameters for Ape-X DQN~\cite{horgan2018distributed}. Parameters left unspecified follow the default ones provided in RLlib v0.13~\cite{liang2018rllib}.}
\label{tab:tiny_abr_train_param}
\end{table}

\begin{table}[t]
\centering
\begin{tabular}{p{4.8cm}p{2.4cm}}
    \toprule
    \textbf{Hyperparameter} & \textbf{Value} \\
    \midrule
    Number of actors & 64\\
    Number of training steps & $1\mathrm{e}{9}$\\
    Learning rate & $7.5\mathrm{e}{-6}$\\
    Replay batch size & 128 \\
    $\gamma$ & 0.95 \\
    Replay buffer size & 2M \\
    N-step return & 7\\
    Value clipping & [-32, 32]\\
    \bottomrule
\end{tabular}
\caption{ Training parameters for the off-policy variant of Gelato. Parameters left unspecified follow the default ones provided in RLlib v0.13~\cite{liang2018rllib}.}
\label{tab:gelato_dqn_parameters}
\end{table}

\label{sec:appendix_abr}
In ABR, we introduce the novel controller architecture Gelato. 

Gelato's neural architecture uses frame-stacking with $10$ past values for the client data, and $5$ future values of chunk sizes and SSIMs at every encoded bitrate. The client data is passed through a $1$D convolution with a kernel size of $3$ and $64$ filters, followed by another $1$D convolution of the same kernel size and filters. The chunk sizes and SSIMs are each passed through their own 1D convolution with a kernel size of $5$ and $32$ filters, each followed by another $1$D convolution with the same kernel size and number of filters. The second layer of convolutions reduces the size of the resulting output by a factor proportional to the size of the kernel. The resulting features are concatenated and passed through a policy and a value network each made up of a single hidden layer of $256$ neurons. Note that the value network is not used outside of training. An inference on Gelato's neural network takes less than $0.35$ ms on average on a core of our $x86-64$ CPU server in Python---a minimal per-chunk overhead for Puffer's $2.002$ second chunk duration. To train Gelato, we use the A2C algorithm~\cite{mnih2016asynchronous} using a standard reward normalization strategy~\cite{pohlen2018observe} and the training parameters defined in Table~\ref{tab:gelato_params}. 

The off-policy DQN variant of Gelato uses the same architecture, swapping the final policy and value networks for a single dueling Q-network made up of a single hidden layer of $256$ neurons. We additionally use a standard reward normalization function~\cite{pohlen2018observe} to normalize the rewards. To train this variant of Gelato, we use the Ape-X DQN algorthm~\cite{horgan2018distributed} using the training parameters defined in Table~\ref{tab:gelato_dqn_parameters}.

We train Pensieve~\cite{pensieve} using its original architecture. However, because the original implementation could only work with the traces provided by the authors, to adapt Pensieve to new traces, we use the same training environment and DRL parameters as Gelato. 

In presenting the results for Gelato in the real world, we re-plot the data found on the Puffer Platform~\cite{fugu} in Figure~\ref{fig:main_results} in Section~\ref{sec:experiments}. In our analysis, we present the data from dates 01 Oct' 2022 through 01 Oct 2023. However, because the platform was experiencing issues and benchmarking other ABR controllers, this data is split across multiple plots. To aggregate the data together, we first download the pre-processed public data available from the Puffer Website~\cite{puffer-online}. Second, we follow the same technique used by the platform, and employ a sampling-based approach to estimate the mean and $95\%$ confidence interval of quality, quality change and stalling for each ABR algorithm. We ignore all the days when the platform was under maintenance (such as 16 January 2023) and days when the platform produced faulty data due to a known bug in the code (such as 21 January 2023). 

For completeness, we present the older results from the Puffer Platform in Figure~\ref{fig:puffer_website_plot} benchmarking the original version of the Fugu controller, which was taken off the platform on 06 October 2022. In this plot, we analyze $25.5$ stream-years of data, collected over the period 07 March 2022 through 05 October 2022. We observe that Gelato-Plume-Static still outperforms the state-of-the-art ABR algorithms in both quality and stalling. This result highlights how Plume can successfully train robust and high performant controllers in simulation, even outperforming in-situ trained controller updated daily. 

\section{Congestion Control Details}
\label{sec:appendix_cc}
In CC, we train and evaluate Aurora~\cite{aurora} with different prioritization techniques. We use framestacking with a history length of 25. We use a 2-layer fully connected neural architecture with 64 units for both the policy and value function. We additionally use State-Dependent noise for exploration~\cite{raffin2022smooth} and reward scaling. We use the training parameters defined in Table~\ref{tab:aurora_params} with the algorithm A2C~\cite{mnih2016asynchronous}.  

\section{Load Balancing Details}
\label{sec:appendix_lb}
In LB, we evaluate different prioritization techniques using standard parameters. We use a 2-layer fully connected neural architecture with [256, 128] units and GeLU activation~\cite{hendrycks2016gaussian} for both the policy and value function. We additionally use reward scaling, and the training parameters defined in Table~\ref{tab:lb_params} with the algorithm PPO~\cite{ppo}. 

\end{document}